\documentclass{article} 

\usepackage[letterpaper, left=1.2truein, right=1.2truein, top = 1.2truein, bottom = 1.2truein]{geometry}

\usepackage{amsmath}
\usepackage{amssymb}
\usepackage{mathtools}
\usepackage{amsthm}
\usepackage{amsfonts}
\usepackage{enumerate}
\usepackage{graphicx}
\usepackage{adjustbox}

\usepackage[blocks, affil-it]{authblk}

\usepackage[round]{natbib}

\usepackage{microtype}
\usepackage{subfigure}
\usepackage{booktabs} 
\usepackage{algorithm}
\usepackage{algorithmic}
\newcommand{\INPUT}{\item[\textbf{input}]}
\newcommand{\OUTPUT}{\item[\textbf{output}]}

\usepackage[CJKbookmarks=true,
            bookmarksnumbered=true,  
			bookmarksopen=true,
			colorlinks=true,
			citecolor=blue,
			linkcolor=blue,
			anchorcolor=red,
			urlcolor=blue]{hyperref}
\usepackage{hyperref}

\usepackage[capitalize,noabbrev]{cleveref}

\theoremstyle{plain}
\newtheorem{theorem}{Theorem}
\crefname{thm}{theorem}{theorems}

\newtheorem{corollary}[theorem]{Corollary}
\theoremstyle{definition}
\newtheorem{definition}[theorem]{Definition}
\newtheorem{assumption}[theorem]{Assumption}
\theoremstyle{remark}

\newcommand\cR{\mathcal{R}}

\newcommand\R{\mathbb{R}}
\renewcommand\P{\mathbb{P}}
\newcommand\E{\mathbb{E}}

\DeclareMathOperator*{\argmin}{argmin}

\newcommand{\iid}{\textnormal{iid}}
\newcommand{\ind}{\textnormal{ind}}
\newcommand{\simiid}{\stackrel{\iid}{\sim}}
\newcommand{\simind}{\stackrel{\ind}{\sim}}

\newcommand{\ca}{\mathcal{A}}
\newcommand{\alg}{\mathcal{A}}
\newcommand{\cd}{\mathcal{D}}
\newcommand{\cq}{\mathcal{Q}}
\newcommand{\cx}{\mathcal{X}}
\newcommand{\cy}{\mathcal{Y}}
\newcommand{\p}{p}
\newcommand{\q}{q}

\newcommand{\seqn}{\textnormal{seq}_{\left[n\right]}}
\newcommand{\seqnn}{\textnormal{seq}_{\left[n-1\right]}}
\newcommand{\pred}{\hat{\cy}}
\newcommand{\Mod}[1]{\ (\mathrm{mod}\ #1)}

\AtBeginDocument{}

\title{Bagging Provides Assumption-free Stability}
\author[1]{Jake A. Soloff}
\author[1]{Rina Foygel Barber}
\author[1,2]{Rebecca Willett}
\affil[1]{Department of Statistics, University of Chicago}
\affil[2]{Department of Computer Science, University of Chicago}
\date{\today}

\begin{document}
\maketitle

\begin{abstract}
Bagging is an important technique for stabilizing machine learning models. In this paper, we derive a finite-sample guarantee on the stability of bagging for any model. 
Our result places no assumptions on the distribution of the data, on the properties of the base algorithm, or on the dimensionality of the covariates. Our guarantee applies to many variants of bagging and is optimal up to a constant. Empirical results validate our findings, showing that bagging successfully stabilizes even highly unstable base algorithms.
\end{abstract}

\section{Introduction}\label{sec-intro}

Algorithmic stability---that is, how perturbing training data influences a learned model---is fundamental to modern data analysis. In learning theory, certain forms of stability are necessary and sufficient for generalization \citep{bousquet2002stability, poggio2004general, shalev2010learnability}. 
In model selection, stability measures can reliably identify important features \citep{meinshausen2010stability,shah2013variable,ren2021derandomizing}.  
In scientific applications, stable methods promote reproducibility, a prerequisite for meaningful inference \citep{yu2013stability}. In distribution-free prediction, stability is a key assumption for the validity of jackknife (that is, leave-one-out cross-validation) prediction intervals \citep{barber2021predictive, steinberger2023conditional}. 

Anticipating various benefits of stability, \citet{breiman1996bagging,breiman1996heuristics} proposed bagging as an ensemble meta-algorithm to stabilize any base learning algorithm. Bagging, short for {\bf b}ootstrap {\bf agg}regat{\bf ing}, refits the base algorithm to many perturbations of the training data and averages the resulting predictions. Breiman's vision of bagging as off-the-shelf stabilizer motivates our main question: \emph{How stable is bagging on an arbitrary base algorithm, placing no assumptions on the data generating distribution?}
In this paper, we first answer this question for the case of base algorithms with bounded outputs and then show extensions to the unbounded case.

\begin{figure}[t]
\begin{center}
\includegraphics[width=.6\columnwidth]{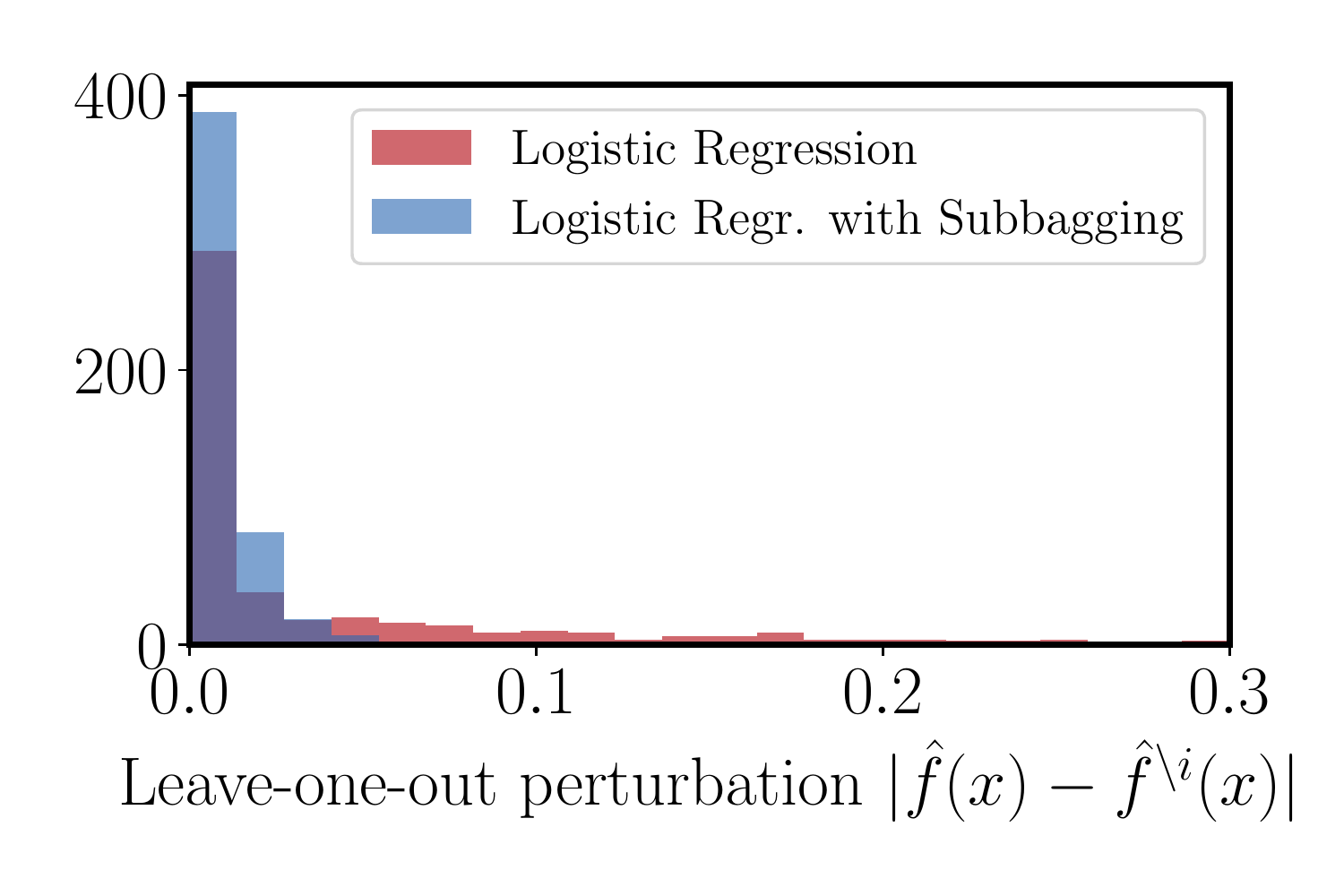}
\vskip -.1in
\caption{Distribution of leave-one-out perturbations for logistic regression (red) and subbagged logistic regression (blue), with $n=500$ and $d=200$. (See Section~\ref{sec-experiments} for details on this simulation.) 
}
\label{fig-double-hist}
\end{center}
\end{figure}
\subsection{Preview of Main Results}

We study the following notion of algorithmic stability:
\begin{definition}[Stability---informal version]\label{def:intro-stability}
    An algorithm is $\left(\varepsilon, \delta\right)$-stable if, for any training data set $\cd$ with $n$ data points, and any test point $x$,
    \begin{align}\label{eq:intro-avg-case-stability}
        \frac{1}{n}\sum_{i=1}^n\mathbf{1}\left\{\left|\hat{f}\left(x\right) - \hat{f}^{\setminus i}\left(x\right)\right| > \varepsilon\right\}\le \delta,
    \end{align}
     where $\hat{f}$ is the model trained on the entire data set $\cd$, while $\hat{f}^{\setminus i}$ is trained on the data set $\cd$ with the $i$th data point removed.
\end{definition}
In other words, this definition requires that, for {\em any} data set, if we drop one training point at random, then the resulting prediction produced by the algorithm is typically insensitive to this perturbation of the training data. 

It is well known that, empirically, bagging and other ensembling procedures tend to improve the stability of an unstable base algorithm. For example, Figure~\ref{fig-double-hist} shows the histograms of  leave-one-out perturbations $\left|\hat{f}\left(x\right) - \hat{f}^{\setminus i}\left(x\right)\right|$ for two different algorithms: logistic regression and logistic regression with subbagging (given by 
$\hat{f}_B\left(x\right):=\frac{1}{B}\sum_{b=1}^B \hat{f}^{\left(b\right)}\left(x\right)$, where each $\hat{f}^{\left(b\right)}$ is a model fitted on $m=n/2$ out of $n$ training data points  sampled at random without replacement). We can clearly see that this perturbation is often far larger for logistic regression than for its subbagged version. 

In this paper, we prove that stability (in the sense of Definition~\ref{def:intro-stability}) is {\em automatically} achieved by the bagged version of any algorithm---with no assumptions on either the algorithm itself or on the training and test data, aside from requiring that the output predictions lie in a bounded range. A special case of our main result can be informally summarized as follows: 
\begin{theorem}[Main result---informal version]\label{thm:informal}
Fix any algorithm with bounded output, and consider its subbagged version with $m$ samples drawn without replacement, 
\[\hat{f}_B\left(x\right):=\frac{1}{B}\sum_{b=1}^B \hat{f}^{\left(b\right)}\left(x\right),\] 
where $B$ is sufficiently large. 
Then the subbagged algorithm satisfies Definition~\ref{def:intro-stability} for any pair $\left(\varepsilon, \delta\right)$ satisfying
\begin{equation}\label{eqn-thm:informal}
\delta\varepsilon^2\gtrsim \frac{1}{n}\cdot\frac{p}{1-p}
\end{equation}
where $p = \frac{m}{n}$.
\end{theorem}

(The formal version of \cref{thm:informal}, including many other forms of bagging, can be found in Section~\ref{sec-guarantees} below. We extend our main result to the unbounded case in \cref{sec-extensions}.) 

In the existing literature, relatively little is known about bagging's stabilization properties without additional assumptions on the base algorithm.\footnote{We defer a more extensive discussion of prior work in this area to \cref{sec-related}.} In this work, our stability guarantees (previewed in \cref{thm:informal}) will:
\begin{itemize}
\item apply to general base algorithms which may be highly unstable,
\item hold for finite sample sizes,
\item provide bounds that are optimal (up to constants), and
\item hold deterministically, allowing for out-of-distribution test points and non-exchangeable data.
\end{itemize}

\section{Algorithmic Stability}\label{sec-stability}

Consider a supervised learning setting with real responses. 
Formally, a learning algorithm~$\ca$ is a function that inputs a data set~$\cd = \left(Z_i\right)_{i=1}^n$ of pairs $Z_i = \left(X_i, Y_i\right)$ of covariates~$X_i\in \cx$ and responses~$Y_i\in \cy$ and an auxiliary random variable $\xi \sim \textnormal{Unif}\left(\left[0,1\right]\right)$ and produces a fitted regression function~$\hat{f} : \cx\to \pred$, given by $\hat{f} = \ca\Big(\cd; \xi\Big)$. 
While many results in the literature consider only symmetric algorithms $\alg$ (i.e., invariant to the ordering of the $n$ training points in $\cd$), here we do not constrain $\alg$ to be symmetric. 

The auxiliary random variable~$\xi$ may be viewed as a random seed, allowing for randomization in $\ca$, if desired.  For example, in many applications, we may wish to optimize an objective such as empirical risk, and then the resulting algorithm $\ca$ consists of the specific numerical operations applied to the training data---for example,~$T$ steps of stochastic gradient descent (SGD) with a specific learning rate and batch size, with the random seed $\xi$ used for drawing the random batches in SGD.
Our notation also allows for deterministic algorithms, since~$\ca$ is free to ignore the input argument $\xi$ and depend only on the data.

There are many ways to define the stability of a learning algorithm. As noted by \citet{shalev2010learnability}, every definition of stability quantifies the sensitivity of the output of $\ca$ to small changes in the training set $\cd$, but they all define `sensitivity of the output' and `small changes in the training set' differently. We present our main results for two definitions of stability and extend our results to many related notions in \cref{sec-alt-frameworks}. One of the strongest possibilities is to require, for all data sets and all test points, that every prediction be insensitive to dropping any single observation. The following definition is closely related to \emph{uniform prediction stability} \citep[see, e.g.,][]{dwork2018privacy}.
\begin{definition}\label{def:worst-case-stability}
    An algorithm~$\ca$ is worst-case $\left(\varepsilon, \delta\right)$-stable if, for all data sets~$\cd = \left(Z_i\right)_{i=1}^n$ and test points~$x\in \cx$,
    \begin{align}\label{eq:worst-case-stability}
        \max_{i\in \left[n\right]}\P_\xi\left\{\left|\hat{f}\left(x\right) - \hat{f}^{\setminus i}\left(x\right)\right| > \varepsilon\right\}\le \delta,
    \end{align}
    where $\hat{f} = \ca\left(\cd; \xi\right)$, $\hat{f}^{\setminus i} = \ca\left(\cd^{\setminus i}; \xi\right)$ and $\cd^{\setminus i} = \left(Z_j\right)_{j\ne i}$. 
\end{definition}

In many settings, however, this requirement is too stringent, since it forces $\ca$ to be stable even when the most influential observation is dropped. A relaxation of this definition is the notion of average-case stability, where the perturbation comes from dropping one observation at random.
Since we are primarily interested in average-case stability in this paper, we refer to it simply as `$\left(\varepsilon, \delta\right)$-stable'. 
\begin{definition}\label{def:average-case-stability} An algorithm~$\ca$ is $\left(\varepsilon, \delta\right)$-stable if, for all data sets~$\cd = \left(Z_i\right)_{i=1}^n$ and test points~$x\in \cx$,
    \begin{align}\label{eq:avg-case-stability}
        \frac{1}{n}\sum_{i=1}^n\P_\xi\left\{\left|\hat{f}\left(x\right) - \hat{f}^{\setminus i}\left(x\right)\right| > \varepsilon\right\}\le \delta,
    \end{align}
    where $\hat{f} = \ca\left(\cd; \xi\right)$, $\hat{f}^{\setminus i} = \ca\left(\cd^{\setminus i}; \xi\right)$ and $\cd^{\setminus i} = \left(Z_j\right)_{j\ne i}$. 
\end{definition}
This is the formal version of Definition~\ref{def:intro-stability}, stated informally earlier---the difference here (aside from introducing notation) lies in the presence of the randomization term $\xi$.

The terminology~`\emph{$\ca$ is $\left(\varepsilon, \delta\right)$-stable}' in Definition~\ref{def:average-case-stability} (or `$\ca$ is worst-case $\left(\varepsilon, \delta\right)$-stable, in Definition~\ref{def:worst-case-stability}) suppresses the dependence on the sample size~$n$. Since we are performing a non-asymptotic stability analysis, we treat~$n$ as a fixed positive integer throughout.

Clearly, Definition~\ref{def:average-case-stability} is implied by Definition~\ref{def:worst-case-stability}, but not vice versa; average-case stability thus relaxes worst-case stability to allow some small fraction of observations to have large leave-one-out perturbation.  Average-case stability is a more permissive condition, and yet it is often sufficient for statistical inference. Indeed, if data points $Z_i$ are exchangeable and~$\ca$ is symmetric, Condition~\eqref{eq:avg-case-stability} implies
\[
\P_{\cd, X_{n+1}, \xi}\left\{\left|\hat{f}\left(X_{n+1}\right) - \hat{f}^{\setminus i}\left(X_{n+1}\right)\right| > \varepsilon\right\}\le \delta
\]
for all $i$ and for a new test point $X_{n+1}$,
which is the condition of ``out-of-sample stability'' used by papers on distribution-free prediction mentioned in Section~\ref{sec-intro} (when $Z_{n+1}=\left(X_{n+1}, Y_{n+1}\right)$ is a new test point that is exchangeable with the training data). Of course, our definition is a stronger property, as it is required to hold uniformly over any training set and any test point. 

In fact,
 worst-case stability ensures an even stronger property---for a training sample $Z_1,\dots,Z_n$, it must hold that
 \[
\P_{\cd,\xi}\left\{\left|\hat{f}\left(X_i\right) - \hat{f}^{\setminus i}\left(X_i\right)\right| > \varepsilon\right\}\le \delta
\]
which is sometimes known as ``in-sample stability''; informally, this bound implies that $\hat{f}$ is not ``overfitted'' to the training data, since its prediction $\hat{f}\left(X_i\right)$ at the $i$th training point is only slightly influenced by $Z_i$. 
On the other hand, average-case stability is not sufficient to ensure this type of bound, even on average over $i=1,\ldots,n$.

In many lines of the literature, it is more standard to define stability with respect to a loss function~$\ell\left(\hat{f}\left(x\right), y\right)$. Define $\left(\varepsilon, \delta\right)$-stability with respect to the loss~$\ell$ as
\begin{align*}
\frac{1}{n}\sum_{i=1}^n\P_\xi\left\{\left|\ell\left(\hat{f}\left(x\right),y\right) - \ell\left(\hat{f}^{\setminus i}\left(x\right), y\right)\right| > \varepsilon\right\}\le \delta.
\end{align*}
If an algorithm $\ca$ is $\left(\varepsilon, \delta\right)$-stable in the sense of~Definition~\ref{def:average-case-stability}, then it is $\left(\varepsilon/L,\delta\right)$-stable with respect to any loss function~$\ell$ that is~$L$-Lipschitz in its first argument. Hence, our stability guarantees immediately apply to any Lipschitz loss.\footnote{Assuming the loss~$\ell$ is Lipschitz is standard in the literature---see, for example, \citet{elisseeff2005stability,hardt2016train}.}

Similarly, in the stability literature, it is more standard to control the expected value of the leave-one-out perturbation~$\left|\hat{f}\left(x\right) - \hat{f}^{\setminus i}\left(x\right)\right|$ rather than controlling a tail probability. However, tail bounds can be easily converted to bounds in expectation using the standard identity
\[
\E\left|\hat{f}\left(x\right) - \hat{f}^{\setminus i}\left(x\right)\right|
= \int_0^\infty \P\left\{\left|\hat{f}\left(x\right) - \hat{f}^{\setminus i}\left(x\right)\right| > \varepsilon\right\}\text{d}\varepsilon.
\] 
In \cref{sec-alt-frameworks}, we define various related notions of stability more formally, and consider the implications of our main result for these alternative definitions of stability.

\section{Bagging and its Variants}\label{sec-main}

Bagging has a rich history in the machine learning literature \citep{breiman1996bagging,dietterich2000ensemble,valentini2002ensembles} and is widely used in a variety of practical algorithms; random forests are a notable example \citep{breiman2001random}. 

The theoretical properties of bagging have also been widely explored.
For example, the stabilization properties of bagging have been 
studied for some specific base algorithms, such as trees \citep{basu2018iterative} or $k$-means \citep{ben2007stability}. \cite{poggio2002bagging} compared the stabilizing properties of bagging to those of ridge regression; \citet{lejeune2020implicit,lejeune2022asymptotics} recently established some deeper connections between the asymptotic risks of bagging and ridge regression (see also \citet{patil2023bagging}). \citet{larsen2023bagging} showed bagging (where the base algorithm is empirical risk minimization) achieves optimal sample complexity for PAC learning in the realizable setting. Additional prior works have addressed the stability of bagging by proving that bagging, under certain conditions, increases the stability of an already stable algorithm \citep{elisseeff2005stability}. In contrast, our results establish stability for bagging when applied to an arbitrary (and possibly highly unstable) base algorithm. We defer a more detailed discussion of prior work in this area to \cref{sec-related}, where we can more fully compare to our own results.

Bagging applies resampling methods to reduce variance, smooth discontinuities, and induce stability in a base algorithm~$\ca$. The meta-algorithm repeatedly samples `bags' from the training data~$\cd$, runs the base algorithm~$\ca$ on each bag, and averages the resulting models. Different resampling methods lead to some common variants:
\begin{itemize}
    \item Classical bagging \citep{breiman1996bagging,breiman1996heuristics} samples $m$ indices with replacement from $\left[n\right]=\left\{1,\dots,n\right\}$. 
    \item Subbagging \citep{andonova2002simple} samples~$m$ indices without replacement from~$\left[n\right]$ (where $m\leq n$).
\end{itemize}

We distinguish `classical bagging,' which employs sampling with replacement, from the more general `bagging,' which we use to refer to any resampling method. For classical bagging, \citet{breiman1996bagging,breiman1996heuristics} originally proposed using the nonparametric bootstrap \citep{efron1979bootstrap}, that is, $m=n$, but $m\ll n$ is often computationally advantageous. 

In both of the above strategies, because there are exactly~$m$ observations in each bag, there is a weak negative correlation between observations (that is, between the event that data point $i$ is in the bag, and that data point $j$ is in the bag). Randomizing the size of each bag is a standard trick that decorrelates these events:

\begin{itemize}
    \item Poissonized bagging \citep{oza2001online, agarwal2014knowing} samples $M$ indices with replacement from~$\left[n\right]$, where~$M\sim \textnormal{Poisson}\left(m\right)$.
    \item Bernoulli subbagging \citep{harrington2003online} samples $M$ indices without replacement from~$\left[n\right]$, where $M\sim \textnormal{Binomial}\left(n, \frac{m}{n}\right)$.
\end{itemize}

Our stability results are quite flexible to the choice of resampling method, and in particular, apply to all four methods described above. In order to unify our results, we now  present a generic version of bagging that includes all four of these variants as special cases.

\subsection{Generic Bagging}

Bagging is a procedure that converts any base algorithm $\ca$ into a new algorithm, its bagged version~$\widetilde\ca_B$.
Define
\[\seqn\coloneqq \left\{\left(i_1, \ldots, i_k\right) : k\ge 0, i_1,\ldots,i_k\in \left[n\right]\right\},\]
which is the set of finite sequences (of any length) consisting of indices in~$\left[n\right]$. We refer to any $r\in\seqn$ as a ``bag''. 
Let $\mathcal{Q}_n$ denote a distribution on $\seqn$. For example, subbagging~$m$ out of~$n$ points corresponds to the uniform distribution over the set of length-$m$ sequences~$\left(i_1,\ldots,i_m\right)$ with distinct entries. Given a bag $r=\left(i_1,\ldots,i_m\right)\in\seqn$ and a data set $\cd=\left(Z_1,\dots,Z_n\right)$, define a new data set $\cd_r= \left(Z_{i_1},\dots,Z_{i_m}\right)$ selecting the data points according to the bag $r$.

\begin{algorithm}
   \caption{Generic Bagging $\widetilde\ca_B$}
   \label{alg:bagging}
    \begin{algorithmic}
   \INPUT Base algorithm $\ca$; data set $\cd$ with $n$ training points; number of bags $B\geq 1$; resampling distribution $\cq_n$
   \FOR{$b=1,\dots,B$}
   \STATE Sample bag $r^{\left(b\right)} = \left(i_1^{\left(b\right)},\ldots,i_{n_b}^{\left(b\right)}\right)\sim\cq_n$
   \STATE Sample seed $\xi^{\left(b\right)} \sim \textnormal{Unif}\left(\left[0,1\right]\right)$
   \STATE Fit model $\hat{f}^{\left(b\right)} = \ca\left(\cd_{r^{\left(b\right)}};\xi^{\left(b\right)}\right)$
   \ENDFOR
   \OUTPUT Averaged model $\hat{f}$ defined by
   \[
   \textstyle
    \hat{f}_B\left(x\right) = \frac{1}{B}\sum_{b=1}^B \hat{f}^{\left(b\right)}\left(x\right)\] 
\end{algorithmic}
\end{algorithm}

To construct the bagged algorithm $\widetilde\ca_B$ using $\ca$ as a base algorithm, we first draw bags $r^{\left(1\right)},\dots,r^{\left(B\right)}$ from the resampling distribution $\cq_n$, then fit a model on each bag using $\alg$, and average the resulting models for the final fitted function. Algorithm~\ref{alg:bagging} summarizes this procedure.

Generic bagging treats the base algorithm~$\ca$ as a `black-box,' in that it only accesses the base algorithm by querying it on different training sets and different random seeds. We write $\widetilde\ca_B$ to denote the resulting algorithm obtained by applying generic bagging with $\alg$ as the base algorithm.

Our theoretical analysis of bagging is simplified by considering an idealized version of generic bagging as the number of bags~$B$ tends to infinity. Our tactic is to directly study the stability of this large-$B$ limit, and then derive analogous results for~$\widetilde\ca_B$ using simple concentration inequalities. To facilitate our theoretical analysis, we define in Algorithm~\ref{alg:bagging-derand} the limiting version of generic bagging.

\begin{algorithm}
   \caption{Derandomized Bagging $\widetilde\ca_\infty$}
   \label{alg:bagging-derand}
    \begin{algorithmic}
   \INPUT Base algorithm $\ca$; data set $\cd$ with $n$ training points; resampling distribution $\cq_n$
   \OUTPUT Averaged model $\hat{f}_\infty$ defined by
   \[\hat{f}_\infty\left(x\right) = \E_{r,\xi}\left[\ca\left(\cd_r; \xi\right)\left(x\right)\right]\]
   where the expectation is taken with respect to $r\sim\cq_n$ and $\xi\sim\textnormal{Unif}\left(\left[0,1\right]\right)$.
\end{algorithmic}
\end{algorithm}

Note that the algorithm~$\ca$ may be a randomized algorithm, but derandomized bagging averages over any randomness in $\ca$ (coming from the random seed $\xi$) as well as the randomness of the bags drawn from $\cq_n$.\footnote{For $\widetilde\ca_\infty$ to be defined, we assume that the expectation $\E_{r,\xi}\left[\ca\left(\cd_r; \xi\right)\left(x\right)\right]$ exists for all data sets $\cd$ and test points~$x$.} For instance, the derandomized form of classical bagging averages uniformly over~$n^m$ possible subsets of the data; in practice,  since we generally cannot afford $n^m$ many calls to $\ca$, we would instead run classical bagging with some large $B$ as the number of randomly sampled bags.

\subsection{The Resampling Distribution \texorpdfstring{$\cq_n$}{}}
To simplify the statement of the main results, we make a symmetry assumption on the resampling method~$\cq_n$. All the variants we have described above (classical bagging, subbagging, Poissonized bagging, Bernoulli subbagging) satisfy this assumption. 

\begin{assumption}\label{assumption:symmetry} The resampling method~$\cq_n$ satisfies 
    \[\cq_n\left\{\left(i_1,\ldots,i_m\right)\right\} = \cq_n\left\{\left(\sigma\left(i_1\right),\ldots,\sigma\left(i_m\right)\right)\right\},\] 
    for all $m$, $i_1,\ldots,i_m\in \left[n\right]$, and permutations $\sigma\in \mathcal{S}_n$. 
\end{assumption}

Intuitively, this symmetry assumption requires the bagging algorithm to treat the indices $\left(1,\ldots,n\right)$ as exchangeable (for example, bags $\left(1,2,2\right)$ and $\left(3,4,4\right)$ are equally likely). 

Different bagging methods attain different degrees of stability. For instance, consider a degenerate case where~$\cq_n$ returns a random permutation of~$\left(1,\ldots,n\right)$ (that is, subbagging with $m=n$). Then~$\hat{f}_\infty$ is simply the result of running the base algorithm on shuffled versions of the data. In this case, the bagged algorithm is only as stable as the base algorithm. Our bounds on the stability of bagging depend on specific parameters of the resampling method~$\cq_n$.

\begin{definition}\label{def:p-q} For~$\cq_n$ satisfying Assumption~\ref{assumption:symmetry}, let
\begin{equation}
    \begin{aligned}
        \p &\coloneqq \P_{r\sim \cq_n}\left\{i\in r\right\}, \\
        \q &\coloneqq -\textnormal{Cov}_{r\sim \cq_n}\left(\mathbf{1}_{i\in r}, \mathbf{1}_{j\in r}\right),
    \end{aligned}
\end{equation}
for any~$i\ne j\in \left[n\right]$.
\end{definition}
Here for a sequence $r=\left(i_1,\ldots,i_m\right)\in\seqn$, we write $i\in r$ to denote the event that $i_k=i$ for some $k$.
Assumption~\ref{assumption:symmetry} ensures that the value of $p$ (and of $q$) are shared across all $i$ (respectively, across all $i\neq j$).

We make the following restrictions on these parameters:
\begin{table*}[t]
    \vskip 0.15in
\begin{center}
\begin{small}
\begin{sc}
\begin{tabular}{lccc}
\toprule
  {\bf Algorithm} & {\bf Resampling method} $\cq_n$ 
  & $\p= \P\left\{i\in r\right\}$ 
  & $\q= -\textnormal{Cov}\left[i\in r, j\in r\right]$ 
  \\
\toprule
 Subbagging & 
 \begin{tabular}{@{}c@{}}
 $r=\left(i_1,\ldots,i_m\right)$ drawn \\
uniformly w/o replacement
 \end{tabular}
& $\frac{m}{n}$ & $\frac{m\left(n-m\right)}{n^2\left(n-1\right)}$ \\ 
\midrule
\begin{tabular}{@{}l@{}}Bernoulli \\ subbagging \end{tabular} &
\begin{tabular}{@{}c@{}}$M\sim \textnormal{Binomial}\left(n, \frac{m}{n}\right)$ \\ $r=\left(i_1,\ldots,i_M\right)$ drawn\\ 
uniformly w/o replacement
\end{tabular}
 & $\frac{m}{n}$ & 0 \\ 
\midrule
 \begin{tabular}{@{}l@{}}Classical\\ bagging\end{tabular} & $r\sim\textnormal{Unif}\left\{\left[n\right]^m\right\}$ 
 & $1-\left(1-\frac{1}{n}\right)^m$ 
 & $\left(1-\frac{1}{n}\right)^{2m} - \left(1-\frac{2}{n}\right)^{m}$ 
  \\ 
\midrule
 \begin{tabular}{@{}l@{}}Poissonized\\ bagging\end{tabular} &  \begin{tabular}{@{}c@{}}$M\sim \textnormal{Poisson}\left(m\right)$ \\ $r\mid M\sim\textnormal{Unif}\left\{\left[n\right]^M\right\}$ \end{tabular}
 & $1-e^{-m/n}$ & 0 \\ 
\bottomrule
\end{tabular}
\end{sc}
\end{small}
\end{center}
\caption{Parameters~$\p$ and~$\q$ from Definition~\ref{def:p-q} for various sampling schemes~$\cq_n$.}\label{tab-parameters}
\vskip 0.1in
\end{table*}

\begin{assumption}\label{assumption:nondegeneracy} $\cq_n$ satisfies~$\p \in \left(0,1\right)$ and~$\q\ge 0$.
\end{assumption}

The constraint~$\p\in \left(0,1\right)$ is a nondegeneracy assumption that guarantees a nonzero probability that any given observation~$Z_i$ gets excluded from some bags and included in others. The constraint~$\q\ge 0$ forces non-positive correlation between observations, that is,~$i\in r$ does not increase the probability of~$j\in r$ for~$i\ne j$.

In our work, the role of the parameter $\q$ on our stability guarantees is always relatively insignificant. The symmetry condition imposed by Assumption \ref{assumption:symmetry} implies that the indicator variables $\left(\mathbf{1}_{i\in r}\right)_{i\in \left[n\right]}$ are exchangeable. Since the covariance matrix of the random vector~$\left(\mathbf{1}_{i\in r}\right)_{i\in \left[n\right]}$ must be positive semidefinite, we always have the upper bound $q \le \frac{p\left(1-\p\right)}{n-1}$.

Table~\ref{tab-parameters} provides values of~$\p$ and~$\q$ for the four different sampling schemes discussed above, which all satisfy Assumption~\ref{assumption:nondegeneracy}.
Classical bagging and subbagging both have a small positive $\q$ due to weak negative correlation between the events $i\in r$ and $j\in r$, while Poissonized bagging and Bernoulli subbagging decorrelate these events and so~$\q = 0$.

Since algorithmic stability compares a model fit on~$n$ observations to a model fit on~$n-1$ observations, we need to specify resampling distributions at both sample sizes, that is, $\cq_n$ and $\cq_{n-1}$; naturally, to guarantee stability, these two distributions must be similar to each other. Specifically, we consider the setting where
\begin{equation}\label{eqn-define-Qn-Qn1}
\textnormal{$\cq_{n-1}$ is given by the distribution of }
\textnormal{$r\sim \cq_n$ conditional on the event $r\in\seqnn$},
\end{equation}
that is, we are conditioning on
the event that the $n$th data point is not contained in the bag. 
For example, if $\cq_n$ is chosen to be subbagging $m$ out of $n$ points (for some fixed $m\leq n-1$), then $\cq_{n-1}$ is equal to the distribution obtained by subbagging $m$ out of $n-1$ points.

\section{Stability Guarantees for Bagging}\label{sec-guarantees}

In this section, we first present our main stability guarantee when the prediction range~$\pred$ is a bounded interval. We then show that this guarantee cannot be improved in general, up to a small multiplicative factor.  

\subsection{Main Result: Guarantee for Average-case Stability}\label{sec-main-result}

We turn to our bound quantifying the average-case stability of derandomized bagging.
In this section, we restrict our attention to settings where the output regression function~$\hat{f}$ is bounded. We consider the unbounded case in Section~\ref{sec-unbounded}.

To examine the stability of  $\widetilde\ca_\infty$ (obtained by applying derandomized bagging to a base algorithm~$\ca$)
our stability results compare the models:
\begin{itemize}
\item $\hat{f}_\infty$, obtained by running derandomized bagging (Algorithm~\ref{alg:bagging-derand}) with base algorithm $\ca$, data set $\cd$, and sampling distribution $\cq_n$; and
\item $\hat{f}_\infty^{\setminus i}$, obtained by running derandomized bagging (Algorithm~\ref{alg:bagging-derand}) with base algorithm $\ca$, data set $\cd^{\setminus i}$, and sampling distribution $\cq_{n-1}$, constructed as in~\eqref{eqn-define-Qn-Qn1}.
\end{itemize}

\begin{theorem}\label{thm:upper} Let~$\pred=\left[0,1\right]$.\footnote{All theoretical results in this section are stated for $\pred=\left[0,1\right]$ for simplicity, but it is straightforward to generalize to the case $\pred=\left[a,b\right]$ by simply replacing $\left(\varepsilon, \delta\right)$-stability with $\left(\varepsilon\left(b-a\right),\delta\right)$-stability in the guarantee.} Fix a distribution~$\mathcal{Q}_n$ on~$\seqn$ satisfying Assumptions~\ref{assumption:symmetry} and~\ref{assumption:nondegeneracy}, and let $\cq_{n-1}$ be defined as in~\eqref{eqn-define-Qn-Qn1}. For any algorithm~$\ca$, derandomized bagging~$\widetilde{\ca}_\infty$ is  $\left(\varepsilon, \delta\right)$-stable provided
\begin{align}\label{eq:eps-delta}
\delta\varepsilon^2
\ge \frac{1}{4n}\left(\frac{\p}{1-\p} + \frac{\q}{\left(1-\p\right)^2}\right).
\end{align}
In particular, since $q\le \frac{p\left(1-\p\right)}{n-1}$, the above bound implies that $\left(\varepsilon, \delta\right)$-stability holds as long as 
\begin{align}
    \delta\varepsilon^2 \ge \frac{1}{4\left(n-1\right)}\cdot\frac{p}{1-p}.
\end{align}
\end{theorem}

We prove \cref{thm:upper}, along with all subsequent results, in Appendix~\ref{appendix-proofs}.\footnote{The proof shows a slightly stronger notion of stability, where $\P_\xi\left\{\left|\hat{f}\left(x\right) - \hat{f}^{\setminus i}\left(x\right)\right| > \varepsilon\right\}$ in \cref{eq:avg-case-stability} is replaced with $\P_\xi\left\{\left|\hat{f}\left(x\right) - \hat{f}^{\setminus i}\left(x\right)\right|\ge \varepsilon\right\}$.}

A simple application of Hoeffding's inequality leads to a similar stability guarantee for generic bagging. In this result, we compare the models:
\begin{itemize}
\item $\hat{f}_B$, obtained by running generic bagging (Algorithm~\ref{alg:bagging}) with base algorithm $\ca$, data set $\cd$, and resampling distribution $\cq_n$; and
\item $\hat{f}_B^{\setminus i}$, obtained by running generic bagging (Algorithm~\ref{alg:bagging}) with base algorithm $\ca$, data set $\cd^{\setminus i}$,  and resampling distribution $\cq_{n-1}$, constructed as in~\eqref{eqn-define-Qn-Qn1}.
\end{itemize}

\begin{theorem}\label{thm:upper-random} Let~$\pred=\left[0,1\right]$. Fix a distribution~$\mathcal{Q}_n$ on~$\seqn$ satisfying Assumptions~\ref{assumption:symmetry} and~\ref{assumption:nondegeneracy}, and let $\cq_{n-1}$ be defined as in~\eqref{eqn-define-Qn-Qn1}. For any algorithm~$\ca$ and any $B\geq 1$, generic bagging~$\widetilde{\ca}_B$ is $\left(\varepsilon + \sqrt{\frac{2}{B}\log\left(\frac{4}{\delta'}\right)}, \delta+\delta'\right)$-stable for any~$\left(\varepsilon, \delta\right)$ satisfying Condition~\eqref{eq:eps-delta} and any $\delta' > 0$.
\end{theorem}

If derandomized bagging is guaranteed to satisfy $\left(\varepsilon, \delta\right)$ stability via~\eqref{eq:eps-delta}, then we may take $B \geq \frac{2}{\varepsilon^2}\log\left(\frac{4}{\delta}\right)$ to guarantee $\left(2\varepsilon, 2\delta\right)$-stability of generic bagging. For instance, if $p\in \left(0,1\right)$ and $\delta\in \left(0,1\right)$ are regarded as constants,~\cref{thm:upper} guarantees stability of derandomized bagging as long as~$\varepsilon \gtrsim \frac{1}{\sqrt{n}}$. In order to guarantee the same level of stability for generic bagging, we need the number of bags~$B$ to be of the same order as the number of observations~$n$, which is typically unrealistic in practice. More generally, for any fixed~$B$, the result accounts for the Monte Carlo error in the generic bagging algorithm. 

\subsubsection{Sampling Regimes for Subbagging}

\cref{thm:upper} covers a wide range of regimes depending on the choices of~$\varepsilon$, $\delta$ and~$p$. In this section, we give some concrete examples in the case of subbagging, to build intuition:\\

\noindent\emph{Proportional subsampling with $m \propto n$:}
    Suppose we employ subbagging with $m=n/2$. The stability condition~\eqref{eq:eps-delta} in the theorem simplifies to~$\delta\varepsilon^2 \ge \frac{1}{4\left(n-1\right)}$. More generally, for $m=O\left(n\right)$, stability holds with $\delta\varepsilon^2 \gtrsim \frac{1}{n}$. Hence, our stability result applies in a variety of regimes. For instance, if~$\delta > 0$ does not depend on~$n$, bagging satisfies average-case $\left(\varepsilon, \delta\right)$-stability with~$\varepsilon = O\left(n^{-1/2}\right)$. We may also take $\varepsilon > 0$ fixed and~$\delta = O\left(n^{-1}\right)$, or even~$\delta = \varepsilon = O\left(n^{-1/3}\right)$ going to zero simultaneously. \\
    
\noindent\emph{Massive subsampling with $m = o\left(n\right)$:}
    For massive data sets, it may be computationally advantageous to subsample a very small fraction of the data \citep{kleiner2014scalable}. \emph{Massive subsampling}, where we take bags of size $m = O\left(n^{\kappa}\right)$ for some $\kappa \in \left(0,1\right)$, can be seen to further enhance stability via our result above. In this case, condition~\eqref{eq:eps-delta} becomes~$\delta\varepsilon^2 \gtrsim \frac{1}{n^{2-\kappa}}$. See \cref{sec-related} for a discussion of results in the literature in this regime. \\
    
\noindent\emph{Minimal subsampling with $m = n-o\left(n\right)$:}
     Massive subsampling, or even subsampling a constant fraction of the data, often comes with some loss of statistical efficiency. To avoid this, our result even allows for resampling schemes with~$m = n-o\left(n\right)$, that is, each subsample contains nearly the entire data set. For example, taking $m=n-n^{\kappa}$ for some $\kappa\in\left(0,1\right)$, condition~\eqref{eq:eps-delta} becomes~$\delta\varepsilon^2 \gtrsim \frac{1}{n^{\kappa}}$. 

\subsection{Tightness of Stability Guarantee}

In the special case of subbagging, we show that \cref{thm:upper} cannot be improved (beyond a constant factor) without assuming more about the base algorithm. We only state this result in the ideal, derandomized case, since this is typically more stable than its finite~$B$ counterpart.

\begin{theorem}\label{thm:lower} Let $\pred = \left[0,1\right]$. Fix $n > m \ge 1$ and $\delta\in \left(0,1/2\right)$. There is a base algorithm~$\ca^\sharp$ such that subbagging~$\widetilde\ca^\sharp_\infty$ with $m$ out of~$n$ observations is not $\left(\varepsilon, \delta\right)$-stable for any
\begin{equation}\label{eqn-thm-lower}
\varepsilon 
<\left(1-\delta-n^{-1}\right)p\,\P\left\{H = \Big\lfloor p\left(1+\lfloor n\delta \rfloor\right)\Big\rfloor\right\},
\end{equation}
where $\p=\frac{m}{n}$ and where the probability is taken with respect to~$H\sim \textnormal{HyperGeometric}\left(n-1, \lfloor n\delta\rfloor, m\right)$.
\end{theorem}
To see how this result compares to the guarantee given in \cref{thm:upper}, consider a simple case where $n\delta$ and $m\delta = np\delta$ are integers, and take $p<1$. Then
\begin{equation*}\P\left\{H = \Big\lfloor p\left(1+\lfloor n\delta \rfloor\right)\Big\rfloor\right\} = \P\left\{H = n p \delta\right\}
= \frac{{n\delta\choose np\delta}\cdot{n\left(1-\delta\right)-1\choose np\left(1-\delta\right)}}{{n-1\choose np}}
\approx \frac{1}{\sqrt{2\pi n \delta\left(1-\delta\right)p\left(1-\p\right)}},
\end{equation*}
where the last step holds by taking Stirling's approximation to each factorial term in each Binomial coefficient (and the approximation is accurate as long as $n\cdot \min\left\{\delta,1-\delta\right\}\cdot\min\left\{p,1-p\right\}$ is large). Thus the right-hand side of~\eqref{eqn-thm-lower} is approximately
\[ \approx \frac{1}{\sqrt{2\pi n}}\cdot \sqrt{\frac{1-\delta}{\delta}\cdot \frac{p}{1-p}}.\]
Since we have assumed $\delta<1/2$, we therefore see that stability fails for $\ca^\sharp$ when (approximately)
\[\delta\varepsilon^2< \frac{1}{4\pi n}\cdot\frac{p}{1-p}.\]
Up to a constant, this matches the leading term of the stability guarantee in \cref{thm:upper}, demonstrating the tightness of our guarantee. 

In \cref{fig-phase}, we plot a phase diagram comparing the stability guarantee~\eqref{eq:eps-delta} with the tightness condition~\eqref{eqn-thm-lower} for finite~$n$. We take $n=500, p = 1/2$, and $q=p\left(1-\p\right)/\left(n-1\right) = 1/1996$, which are the values of $p$ and $q$ for subbagging with $m=n/2$ (see \cref{tab-parameters}). The blue line shows, for each $\delta$, the minimum $\varepsilon$ satisfying \eqref{eq:eps-delta}, and the shaded blue region shows additional $\left(\varepsilon, \delta\right)$ pairs satisfying the inequality. This means that, for any base algorithm $\ca$ with outputs in $\pred =\left[0,1\right]$, its subbagged version is guaranteed to satisfy $\left(\varepsilon, \delta\right)$-stability for any pair $\left(\varepsilon, \delta\right)$ in the blue shaded region.
Similarly, the red line shows, for each $\delta$, the maximum $\varepsilon$ satisfying
\eqref{eqn-thm-lower}. This means that, for any $\left(\varepsilon, \delta\right)$ in the red shaded region, we can construct an algorithm $\ca^\sharp$, again with outputs in $\pred=\left[0,1\right]$, such that its subbagged version fails to be $\left(\varepsilon, \delta\right)$-stable. The narrow white region between the two conditions illustrates the small gap between the two results.

\begin{figure}[tb]
\begin{center}
\includegraphics[width=.6\linewidth]{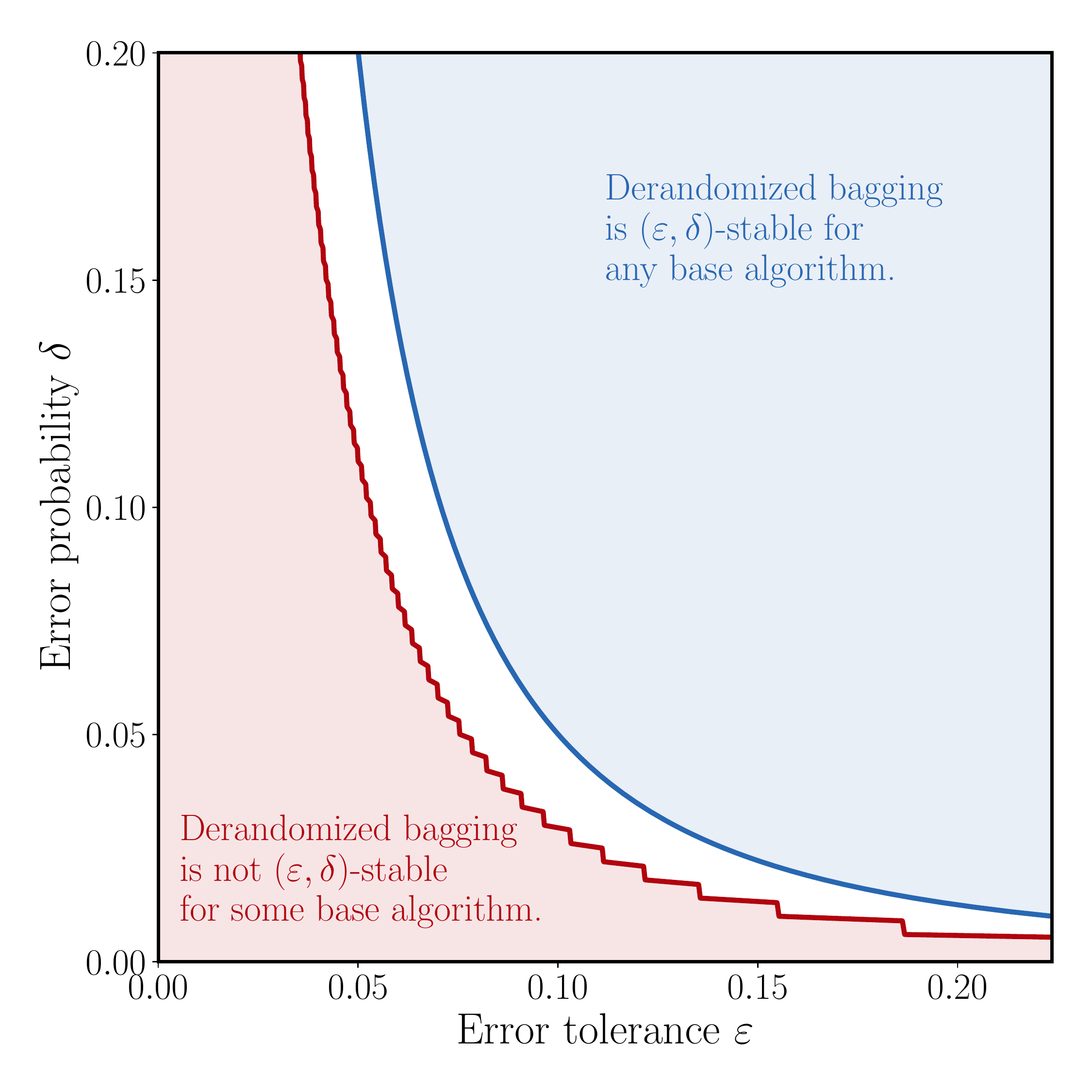}
\caption{Phase diagram comparing \cref{thm:upper,thm:lower}, with $n=500, p=0.5$.}
\label{fig-phase}
\end{center}
\end{figure}

\section{Extensions}\label{sec-extensions}

In this section, we consider various extensions of our main result. We first discuss two approaches to the case of unbounded outputs. Next, we show a hardness result explaining why we cannot obtain a similar guarantee for worst-case stability. Finally, we consider the implications of our main result for various alternative definitions of stability.

\subsection{Unbounded Outputs}\label{sec-unbounded}

We next extend our main result to algorithms~$\ca$ with unbounded output~$\pred = \R$. For derandomized bagging~$\widetilde\ca_\infty$ to be well-defined, we assume that the expectation
\[\E_{r,\xi}\left[\ca\left(\cd_r; \xi\right)\left(x\right)\right]\]
exists. For instance, for classical bagging, subbagging, and Bernoulli subbagging, the average over $r\sim \cq_n$ constitutes a finite sum, so we are simply assuming that expectation over the random seed
\[\E_{\xi}\left[\ca\left(\cd; \xi\right)\left(x\right)\right]\]
exists for any fixed data set $\cd$. 

In order to establish some control over the scale of the outputs of the fitted model, we extend our definition of average-case stability to allow for a data-dependent component. Consider for instance any algorithm $\ca$ with $\pred = \left[0,1\right]$, and define a new algorithm $\ca'$ scaling the outputs by $R > 0$, that is, $\ca'\left(\cd; \xi\right) = R\cdot\ca\left(\cd; \xi\right)$. If the original algorithm $\ca$ is $\left(\varepsilon, \delta\right)$-stable, then the scaled algorithm~$\ca'$ is $\left(\varepsilon R, \delta\right)$-stable. 

We might hope that we can take $R$ to be the empirical range of the algorithm,
\[R = \textnormal{Range}\left(\cd, x\right) = \sup_{r : \cq_n\left(\left\{r\right\}\right) > 0}\E_\xi\left[\ca\left(\cd_{r}; \xi\right)\left(x\right)\right]
- \inf_{r : \cq_n\left(\left\{r\right\}\right) > 0}\E_\xi\left[\ca\left(\cd_{r}; \xi\right)\left(x\right)\right].\]
However, since this quantity depends (in general) on $\cd$ and on $x$, it would not be well-defined to claim that $\ca$ is $\left(\varepsilon R, \delta\right)$-stable universally across all $\cd$ and all $x$.

Instead, to allow for a data-dependent range, we consider scaling~$\varepsilon$ by a data-dependent scale parameter~$\cR\left(\cd, x\right)$, where
\[
\cR : \bigcup_{n\ge 0} \left(\cx\times\cy\right)^n \times \cx \to \R_+.
\]
We now define $\left(\varepsilon, \delta, \cR\right)$-stability to account for data-dependent changes in scale.

\begin{definition} Let $\varepsilon, \delta \ge 0$ and let ${\cR}$ denote a data-dependent range (formally defined above). An algorithm~$\ca$ is $\left(\varepsilon, \delta, \cR\right)$-stable if, for all data sets~$\cd = \left(Z_i\right)_{i=1}^n$ of size $n$ and all test points~$x\in \cx$,
    \begin{align}
        \frac{1}{n}\sum_{i=1}^n\P_\xi\left\{\left|\hat{f}\left(x\right) - \hat{f}^{\setminus i}\left(x\right)\right| > \varepsilon\,{\cR}\left(\cd, x\right)\right\}\le \delta,
    \end{align}
    where $\hat{f} = \ca\left(\cd; \xi\right)$, $\hat{f}^{\setminus i} = \ca\left(\cd^{\setminus i}; \xi\right)$ and $\cd^{\setminus i} = \left(Z_j\right)_{j\ne i}$. 
\end{definition}

Inspecting the proof of \cref{thm:upper}, we only use boundedness to control the variance of our model predictions $\E_\xi\left[\ca\left(\cd_{r}; \xi\right)\left(x\right)\right]$ as a function of the random bag $r\sim \cq_n$. This observation leads to the following, more general result.

\begin{theorem}\label{thm:unbounded-general} Let~$\pred=\R$. Fix a distribution~$\mathcal{Q}_n$ on~$\seqn$ satisfying Assumptions~\ref{assumption:symmetry} and~\ref{assumption:nondegeneracy}, and let $\cq_{n-1}$ be defined as in~\eqref{eqn-define-Qn-Qn1}. Let $\left(\varepsilon, \delta\right)$ satisfy \cref{eq:eps-delta}. For any algorithm~$\ca$, derandomized bagging~$\widetilde{\ca}_\infty$ is $\left(\varepsilon, \delta, {\cR}^*\right)$-stable, where 
\begin{align}\label{eq:radius} 
{\cR}^*\left(\cd, x\right) := 2\sqrt{\textnormal{Var}_{r\sim \cq_n}\Big(\E_\xi\left[\ca\left(\cd_{r}; \xi\right)\left(x\right)\right]\Big)} \leq \textnormal{Range}\left(\cd, x\right).
\end{align}
\end{theorem}

As long as $\E_\xi\left[\ca\left(\cd_{r}; \xi\right)\left(x\right)\right]$ is well-defined for every $r$, the range ${\textnormal{Range}\left(\cd, x\right)}$ is automatically finite for any $\cq_n$ with finite support. Furthermore, \cref{thm:unbounded-general} strictly generalizes the stability guarantee of \cref{thm:upper}, since ${\cR}^*\left(\cd, x\right)\le \textnormal{Range}\left(\cd, x\right)\le 1$ in the case $\pred =\left[0,1\right]$. We present a weaker result for the finite-$B$ regime in \cref{sec-finite-B-unbdd}.

\subsubsection{Alternative Approach: Adaptive Clipping}

In some settings, for example, with heavy tailed responses, the range in the previous display or the standard deviation in~\cref{eq:radius} may be prohibitively large. One way to reduce the standard deviation ${\cR}^*\left(\cd, x\right)$ is to post-process the algorithm~$\ca$. We next consider the advantages of clipping the output of~$\ca$ to secure greater stability.

Given an interval $I = \left[l,u\right]$ and a response~$\hat{y}\in \R$, the clipped response~$\textnormal{Clip}_I\left(\hat{y}\right)$ is defined as
\begin{align}
    \textnormal{Clip}_I\left(\hat{y}\right)
    \coloneqq \max\left\{l, \min\left\{\hat{y}, u\right\}\right\}.
\end{align}
In Algorithm~\ref{alg:bagging-clipped}, we define a variant of the derandomized bagging algorithm that allows the individual bagged predictions~$\hat{f}_\infty^{\left(r\right)}\left(x\right)$ to be clipped to some interval~$I = I\left(\cd\right)$ that depends on the full data set. We write~$\widetilde\ca_{B, I}$ to denote the algorithm obtained by applying adaptively clipped bagging with~$\ca$ as the base algorithm. Stability of \cref{alg:bagging-clipped} does not follow immediately from \cref{thm:unbounded-general} because $I\left(\cd\right)$ may not be the same as $I\left(\cd^{\setminus i}\right)$, so the algorithm being bagged is itself changing when we perturb the training data. For simplicity, we state our result for the derandomized limit~$\widetilde\ca_{\infty, I}$ and give a finite~$B$ version in~\cref{sec-finite-B-unbdd}.

\begin{algorithm}
   \caption{Adaptively Clipped Bagging}
   \label{alg:bagging-clipped}
    \begin{algorithmic}
   \INPUT Base algorithm $\ca$; data set $\cd$ with $n$ training points; number of bags $B\geq 1$; resampling distribution $\cq_n$; data-dependent range $I\left(\cdot\right)$
   \FOR{$b=1,\dots,B$}
    \STATE Sample bag $r^{\left(b\right)} = \left(i_1^{\left(b\right)},\ldots,i_{n_b}^{\left(b\right)}\right)\sim\cq_n$
   \STATE Sample seed $\xi^{\left(b\right)} \sim \textnormal{Unif}\left(\left[0,1\right]\right)$
   \STATE Fit model $\hat{f}^{\left(b\right)} = \ca\left(\cd_{r^{\left(b\right)}};\xi^{\left(b\right)}\right)$
   \ENDFOR
   \OUTPUT Averaged model $\hat{f}_{B,I}$ defined by
   \[
   \textstyle
    \hat{f}_{B,I}\left(x\right) = \frac{1}{B}\sum_{b=1}^B \textnormal{Clip}_{I\left(\cd\right)}\left(\hat{f}^{\left(b\right)}\left(x\right)\right)\] 
\end{algorithmic}
\end{algorithm}

\begin{theorem}\label{thm:unbounded}
    Let~$\pred=\R$. Fix a distribution~$\mathcal{Q}_n$ on~$\seqn$ satisfying Assumptions~\ref{assumption:symmetry} and~\ref{assumption:nondegeneracy}, and let $\cq_{n-1}$ be defined as in~\eqref{eqn-define-Qn-Qn1}. Suppose the mapping~$I\left(\cdot\right)$ from data sets to intervals satisfies
    \begin{align}\label{eq:interval-stability}
    \frac{1}{n}\sum_{i=1}^n \mathbf{1}\left\{I\left(\cd\right)\ne I\left(\cd^{\setminus i}\right)\right\} \le \delta_I.
    \end{align}
    Let ${\cR}\left(\cd, x\right) = \textnormal{length}\left(I\left(\cd\right)\right)$ and let $\left(\varepsilon, \delta\right)$ satisfy \cref{eq:eps-delta}. For any algorithm~$\ca$, derandomized adaptively clipped bagging~$\widetilde{\ca}_{\infty,I}$ is $\left(\varepsilon, \delta + \delta_I, {\cR}\right)$-stable.
\end{theorem}

As a special case, consider taking $I\left({\cd}\right)$ to be the observed range, that is, \[I\left(\cd\right) = \left[\min_i Y_i, \max_i Y_i\right].\] When $\left(Z_i\right)_{i=1}^{n+1}$ are exchangeable random variables, we can apply \cref{thm:unbounded} with $\delta' = \frac{2}{n}$. Restricting to the empirical range of the $Y_i$'s does not substantially limit the learned regression function~$\hat{f}_{\infty,I}$---it simply requires that predictions cannot lie outside the observed range of the training data (which is already satisfied by many base algorithms, such as nearest neighbors or regression trees, and typically would not substantially alter the output of many other algorithms).
More generally, we can take
$I\left(\cd\right) = \left[Y_{\left(k\right)}, Y_{\left(n+1-k\right)}\right]$ for some fixed $k<n/2$, where $Y_{\left(1\right)}\le\cdots\le Y_{\left(n\right)}$ denote the order statistics of $Y_1,\dots,Y_n$. In this case, we have~$\delta' = \frac{2k}{n}$ in~\eqref{eq:interval-stability}, which allows for some fraction~$\delta'$ of outliers to be removed when constructing the data-dependent range, thus ensuring that ${\cR}\left(\cd, x\right)$ is not too large.

{Of course, there are many other potential strategies for defining the data-dependent range $I\left(\cd\right)$, and the benefits and drawbacks of these various choices depend on the specific data distribution and base algorithm. Exploring these options, and designing practical versions of this procedure to provide accurate fitted models with meaningful stability guarantees, is an important question for future work.}

\subsection{Hardness of Worst-case Stability}\label{sec-worst-case}

Our results above establish that $\left(\varepsilon, \delta\right)$-stability can be guaranteed for any (bounded) base algorithm even for very small $\varepsilon$---for instance, taking $p\in\left(0,1\right)$ to be a constant, we can choose $\varepsilon= O\left(n^{-1/2}\right)$. {Next, we show} that no analogous result exists for worst-case stability---indeed, for this stricter definition, stability cannot be guaranteed for any $\varepsilon<p$, and therefore $\varepsilon = O\left(n^{-1/2}\right)$ can only be guaranteed via massively subsampling the data with $p=O\left(n^{-1/2}\right)$.

\begin{theorem}\label{thm:worst} Fix~$\cq_n$ and let $\pred=\left[0,1\right]$.
\begin{itemize}
    \item[(i)] For any algorithm~$\ca$, derandomized bagging is worst-case~$\left(\p, \delta\right)$-stable for all $\delta$.
    \item[(ii)] If $\left|\cx\right| > 1$, there is a base algorithm~$\ca^\dagger$ such that derandomized bagging is not worst-case $\left(\varepsilon, \delta\right)$-stable for any $\varepsilon < \p$ and $\delta<1$.
\end{itemize}
\end{theorem}

Part~\emph{(i)} of the theorem has repeatedly appeared in various forms (see, e.g., \citet[][Theorem~3.1]{poggio2002bagging}, \citet[Proposition~4.3]{elisseeff2005stability} and \citet[Theorem~5]{chen2022debiased}); in Section~\ref{sec-related} we discuss how this observation has led some authors on algorithmic stability to advocate for subsampling a decreasing fraction of the data~$m=o\left(n\right)$ as~$n\to\infty$. In contrast, by moving to average-case stability, our results allow $m=O\left(n\right)$ and even $m=n-o\left(n\right)$, enabling far greater accuracy in the fitted models.

The base algorithm~$\ca^\dagger$ in the proof of part \emph{(ii)} of the theorem memorizes the training data:
\[
\ca^\dagger\left(\cd\right)\left(x\right) := \mathbf{1}\left\{\exists \left(\tilde{x}, \tilde{y}\right)\in \cd : \tilde{x} = x\right\}.
\]
If $x=x_i$ for precisely one $i\in \left[n\right]$, then this training point~$\left(x_i, y_i\right)$ has maximal influence on the value of~$\hat{f}_\infty\left(x\right)$---every bag containing~$\left(x_i, y_i\right)$ predicts 1, and every bag not containing~$\left(x_i, y_i\right)$ leads to a predicts 0. This counterexample can be used to show an even stronger hardness result, for average-case, ``in-sample'' stability (discussed earlier in \cref{sec-stability}): if $x_1,\ldots,x_n$ are all distinct,
\[
\frac{1}{n}\sum_{i=1}^n 1\left\{\left|\hat{f}_\infty^\dagger\left(x_i\right) - \hat{f}_\infty^{\dagger\setminus i}\left(x_i\right)\right| > \varepsilon\right\} = \begin{cases}
1 & \text{if }\varepsilon \le p \\
0 & \text{if }\varepsilon > p
\end{cases},
\]
where $\hat{f}_\infty^\dagger = \ca_\infty^\dagger\left(\cd\right)$. Note, however, that for a fixed $x$, \emph{at most one} index~$i\in \left[n\right]$ can change the bagged prediction by~$p$. This limitation of~$\ca^\dagger$ provides useful intuition for why we may expect a stronger result for our main definition of $\left(\varepsilon, \delta\right)$-stability, Definition~\ref{def:average-case-stability}, where the test point~$x\in \cx$ is fixed. 

\subsection{Alternative Frameworks for the Main Result}\label{sec-alt-frameworks}

In this section, we discuss various implications of our main stability guarantee for related criteria.

\subsubsection{Stability in Expectation}

In Definition~\ref{def:average-case-stability}, average-case algorithmic stability controls the tail of the distribution of leave-one-out perturbations. Some authors~\citep[e.g.,][]{bousquet2002stability,elisseeff2005stability} prefer to work with the expected value of the leave-one-out perturbation. We can consider a version of average-case stability that works with expected values rather than probabilities, requiring that
\begin{align}\label{eq:def-oos-prediction-stability}
    \frac{1}{n}&\sum_{i=1}^n\E_{\xi}\left|\hat{f}\left(x\right) - \hat{f}^{\setminus i}\left(x\right)\right| \le \beta,
\end{align}
for all data sets~$\cd$ of size~$n$ and test points $x\in \cx$, where $\hat{f} = \ca\left(\cd; \xi\right)$ and $\hat{f}^{\setminus i} = \ca\left(\cd^{\setminus i}; \xi\right)$.

\begin{corollary}\label{cor:expectation} In the setting of \cref{thm:upper}, for any $B$, generic bagging~$\widetilde\ca_B$ satisfies stability condition~\eqref{eq:def-oos-prediction-stability} at level $\beta = \beta_{n,m,B}$, where
\begin{equation}\label{eq:prediction-stability-level}
\begin{aligned}
    \beta_{n,m,B} = \sqrt{\frac{1}{4n}\left(\frac{\p}{1-\p} + \frac{\q}{\left(1-\p\right)^2}\right)} + \sqrt{\frac{2\pi}{B}}.
\end{aligned}
\end{equation}
\end{corollary}

Note that the scaling of~$\beta$ in this result is comparable to the scaling of $\varepsilon$ in \cref{thm:upper-random} if we take $\delta,\delta'$ to be constant. The result holds for any~$B$, including the case of derandomized bagging by taking $B\to\infty$.

\subsubsection{Stability in the Loss}\label{sec:loss-stability}

Building on earlier definitions of stability \citep{kearns1999algorithmic, bousquet2002stability}, \citet[][Definition 7]{elisseeff2005stability} say that a randomized algorithm~$\ca$ satisfies \emph{random hypothesis stability} at level $\beta$ with respect to the loss function~$\ell$ and distribution~$P$ if the following holds:
\begin{align}\label{eq:hypothesis-stability}
\forall i\in\left\{1,\ldots,n\right\},~\E_{\left(X_i, Y_i\right)_{i=1}^{n+1}\simiid P, \xi}\left|\ell\left(\hat{f}\left(X_{n+1}\right), Y_{n+1}\right) - \ell\left(\hat{f}^{\setminus i}\left(X_{n+1}\right), Y_{n+1}\right)\right| \le \beta,
\end{align}
where $\hat{f} = \ca\left(\cd; \xi\right)$, $\hat{f}^{\setminus i} = \ca\left(\cd^{\setminus i}; \xi\right)$, $\cd = \left(X_i, Y_i\right)_{i=1}^n$. Our next result records the straightforward observation that Corollary~\ref{cor:expectation} implies random hypothesis stability with respect to any loss~$\ell$ that is Lipschitz in its first argument.

\begin{corollary}\label{cor:loss-stability} Let~$\pred=\left[0,1\right]$. Fix a distribution~$\mathcal{Q}_n$ on~$\seqn$ satisfying Assumptions~\ref{assumption:symmetry} and~\ref{assumption:nondegeneracy}, and let $\cq_{n-1}$ be defined as in~\eqref{eqn-define-Qn-Qn1}. Let $P$ denote any distribution on $\cx\times \cy$, and let $\ell : \pred\times\cy \to \R_+$ denote any loss function that is $L$-Lipschitz in its first argument. For any algorithm~$\ca$ and any $B$, generic bagging~$\widetilde{\ca}_B$ satisfies random hypothesis stability~\eqref{eq:hypothesis-stability} at level~$\beta = L\beta_{n,m,B}$, where $\beta_{n,m,B}$ is defined as in~\cref{eq:prediction-stability-level}.
\end{corollary}

In fact, our main result implies that the inequality in \eqref{eq:hypothesis-stability} holds (on average over~$i\in \left[n\right]$) even \emph{conditional on the training data~$\cd$ and the test point~$\left(X_{n+1}, Y_{n+1}\right)$}, eliminating the assumption that the data are iid---in fact, in our result, the test point can be adversarially chosen. 

\subsubsection{Replace-one Stability}

The stability definitions in this paper concern the leave-one-out perturbation $\left|\hat{f}\left(x\right) - \hat{f}^{\setminus i}\left(x\right)\right|$. Alternative definitions, used for example by \citet{shalev2010learnability}, are obtained by considering a `replace-one' perturbation $\left|\hat{f}\left(x\right) - \hat{f}^{\left(i\right)}\left(x\right)\right|$, where 
\[
\hat{f}^{\left(i\right)} = \ca\left(\cd^{\left(i\right)}; \xi\right) 
\qquad\text{and}\qquad
\cd^{\left(i\right)} = \cd^{\setminus i} \cup \left(Z_i'\right).
\]
 We say that a randomized algorithm~$\ca$ satisfies \emph{random replace-one hypothesis stability} $\beta$ with respect to the loss function~$\ell$ and distribution~$P$ if the following holds:
 \begin{align}
     \forall i\in\left\{1,\ldots,n\right\},~\E_{\left(X_i, Y_i\right)_{i=1}^{n},\left(X_i', Y_i'\right)\simiid P, \xi}\left|\ell\left(\hat{f}\left(x\right), y\right) - \ell\left(\hat{f}^{\left(i\right)}\left(x\right), y\right)\right| \le \beta.
 \end{align}
 Stability to leave-one-out perturbations is typically stronger than stability to replace-one perturbations. To see this, note that, by the triangle inequality, the replace-one perturbation can be bounded as
\[
\left|\hat{f}\left(x\right) - \hat{f}^{\left(i\right)}\left(x\right)\right|
\le \left|\hat{f}\left(x\right) - \hat{f}^{\setminus i}\left(x\right)\right| + \left|\hat{f}^{\left(i\right)}\left(x\right) - \hat{f}^{\setminus i}\left(x\right)\right|,
\]
where both terms on the right-hand-side are leave-one-out perturbations. A guarantee for replace-one stability thus follows immediately from Corollary~\ref{cor:expectation}.

\begin{corollary} In the setting of Corollary~\ref{cor:loss-stability}, generic bagging satisfies random replace-one hypothesis stability at level~$\beta = 2L\beta_{n,m,B}$, where $\beta_{n,m,B}$ is defined in \cref{eq:prediction-stability-level} and $\hat{f}_B^{\left(i\right)}$ is obtained by running generic bagging (Algorithm~\ref{alg:bagging}) with base algorithm $\ca$, data set $\cd^{\left(i\right)}$, and resampling distribution~$\cq_n$.
\end{corollary}

\section{Experiments}\label{sec-experiments}

In this section, we study the stability of subbagging in simulation experiments. We use scikit-learn \citep{scikit-learn} for all base algorithms. Code to reproduce all experiments is available at \url{https://github.com/jake-soloff/subbagging-experiments}.

\subsection{Data and Methods}

We consider four simulation settings:
\begin{itemize}
\item {\bf Setting 1:} We simulate from the following data generating process:
\[  X_i \simiid \mathcal{N}\left(0, I_d\right), \ Y_i\mid X_i \simind \textnormal{Bernoulli}\left(\frac{1}{1+\exp\left(-X_i^\top\theta^*\right)}\right),\]
with sample size $n=500$ and dimension~$d=200$, and where~$\theta^* = \left(.1,\ldots,.1\right)\in\R^d$. The base algorithm $\ca$ is the output of $\ell_2$-regularized logistic regression, given by $\ca\left(\cd\right)\left(x\right) = \hat{f}_{\hat\theta}\left(x\right) := \left(1+e^{-x^\top\hat\theta}\right)^{-1}$, where
\[
\hat\theta = \argmin_{\theta\in \R^d} \left\{C\sum_{i=1}^n\left(-Y_i\log\left(\hat{f}_\theta\left(X_i\right)\right)-\left(1-Y_i\right)\log\left(1-\hat{f}_\theta\left(X_i\right)\right)\right) + \frac{1}{2}\left\|\theta\right\|_2^2\right\}.
\]
We use \verb+sklearn.linear_model.LogisticRegression+, setting options \verb+penalty=`l2'+, \verb+C=1e3/n+ and \verb+fit_intercept=False+, leaving all other parameters at their default values.
\item {\bf Setting 2:} Same as Setting 1, changing only the sample size to $n=1000$.
\item {\bf Setting 3:} Same as Setting 1, changing only the base algorithm $\ca$ to a neural network with a single hidden layer. We use \verb+sklearn.neural_network.MLPClassifier+, setting \verb+hidden_layer_sizes=(40,)+, \verb+solver="sgd"+, \verb+learning_rate_init=0.2+, \verb+max_iter=8+, and \verb+alpha=1e-4+, leaving all other parameters at their default values.
\item {\bf Setting 4:} We simulate from the following data generating process:
\begin{align*}
    \left(X_i,\alpha_i,\gamma_i\right) &\simiid \textnormal{Unif}\left(\left[0,1\right]^d\right)\times \textnormal{Unif}\left(\left[-.25, .25\right]\right)\times \textnormal{Unif}\left(\left[0,1\right]\right),\\
    Y_i &= \sum_{j=1}^d \sin\left(\frac{X_{ij}}{j}\right) + \alpha_i \mathbf{1}\left\{i = 1\Mod 3\right\} + \gamma_i \mathbf{1}\left\{i = 1\Mod 4\right\},
\end{align*}
with $n=500$ and~$d = 40$. Note that the algorithm has access to the observed data~$\mathcal{D} = \left(X_i, Y_i\right)_{i=1}^n$, that is, $\alpha_i$ and $\gamma_i$ are latent variables used only to generate the data~$\mathcal{D}$. We apply \verb+sklearn.tree.DecisionTreeRegressor+ to train the regression trees, setting \verb+max_depth=50+ and leaving all other parameters at their default values. 
\end{itemize}

\begin{figure}[p]
\begin{center}
\includegraphics[width=.98\linewidth]{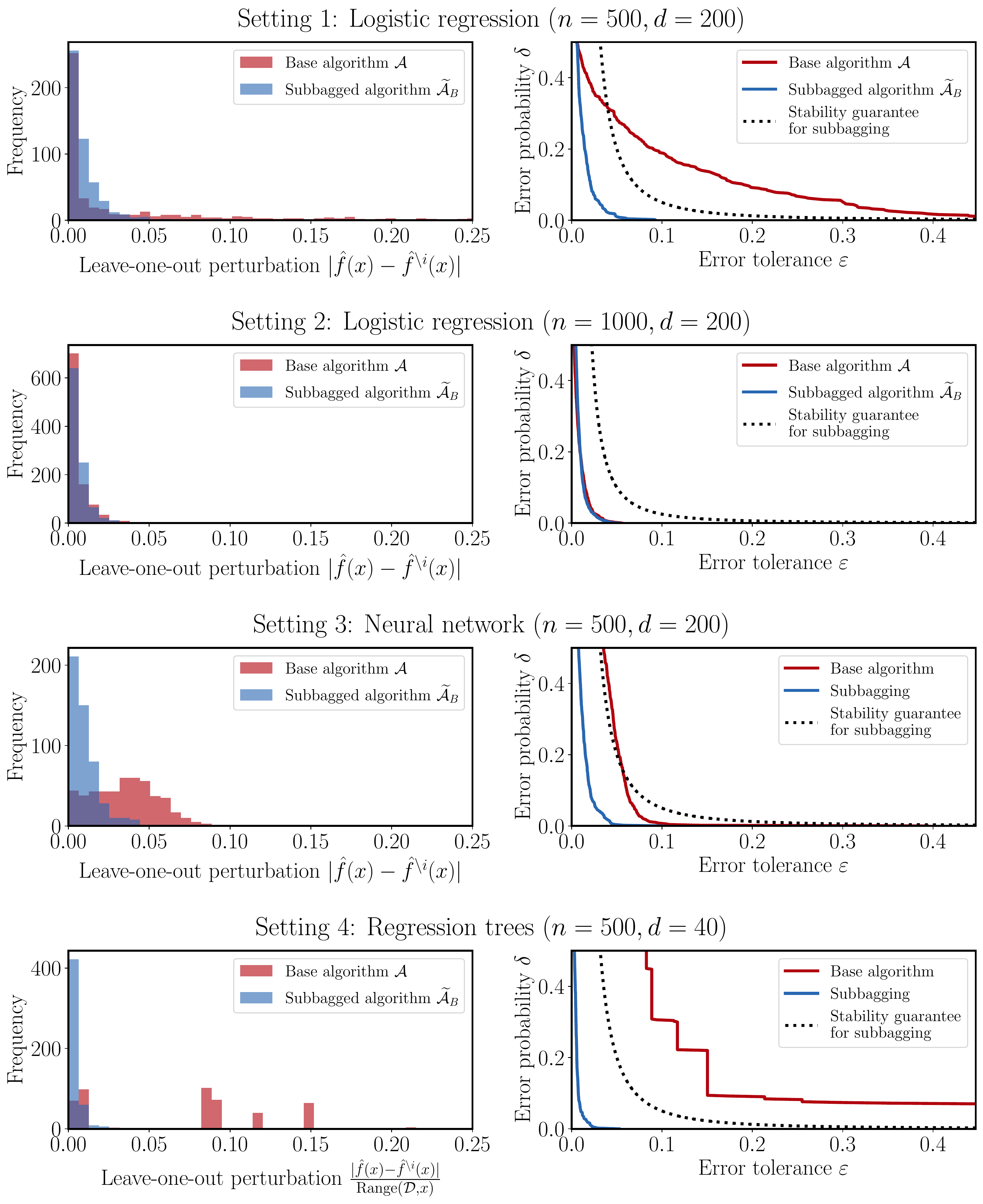}
\vskip -.05in
\caption{Simulation results comparing the stability of subbagging~$\widetilde\ca_B$ to that of the corresponding base algorithm~$\ca$. Left: Histogram of leave-one-out perturbations. Right:  for each $\varepsilon$, the smallest~$\delta$ such that the algorithm is~$\left(\varepsilon, \delta\right)$-stable in the sense of Definition~\ref{def:average-case-stability}. Higher curves thus represent greater instability. In all settings, $m=n/2$ and $B = 10000$.
\vspace{-.2in}
}
\label{fig-experiments}
\end{center}
\end{figure}

\subsection{Results}

Our results are shown in \cref{fig-experiments}. In each setting, we apply the base algorithm~$\ca$ as well as subbagging~$\widetilde\ca_B$ with~$m=n/2$ samples in each bag, using $B=10000$ bags. The left panels of \cref{fig-experiments} show the histogram of leave-one-out perturbations $\left|\hat{f}\left(x\right) - \hat{f}^{\setminus i}\left(x\right)\right|$ for $i\in\left\{1,\ldots,n\right\}$. In the right panels of \cref{fig-experiments}, for a fixed data set and algorithm, we measure stability by plotting, for each value of $\varepsilon$, the smallest value of $\delta$ such that the algorithm is $\left(\varepsilon, \delta\right)$-stable: 
\[\delta = \frac{1}{n}\sum_{i=1}^n1_{\left|\hat{f}\left(x\right)-\hat{f}^{\setminus i}\left(x\right)\right| > \varepsilon}.\]
In each case, the test point~$x=X_{n+1}$ is generated from the same distribution as $X_1,\ldots,X_n$.

For logistic regression (Settings 1 and 2), we see that the subbagged algorithm is highly stable for both values of $n$---in particular, the blue curves lie below the black dotted line, showing that subbagged logistic regression satisfies the theoretical guarantee of \cref{thm:upper}. By contrast, for $n=500$ and $d=200$ (Setting 1), the red curve lies much higher in the plot, showing greater instability; this reveals that the base algorithm, logistic regression (with extremely small regularization), is highly unstable in this regime \citep[see, e.g.,][]{10.1214/18-AOS1789}. For $n=1000$ and $d=200$ (Setting 2), on the other hand, we see that the base algorithm is quite stable---indeed, in this setting, each bag is highly unstable (since $m = n/2 = 500$), but the stability of subbagging is still comparable to that of the base algorithm. These first two settings illustrate our theory by showing that the subbagged algorithm satisfies the stability guarantee regardless of whether the base algorithm is stable. 

In Setting 3, we repeat the same experiment where the base algorithm is a neural network. The neural network base algorithm slightly violates the stability guarantee, and in this case, subbagging improves the stability.

In Setting 4, we simulate from a more complex data generating process. We again see that the subbagged algorithm is highly stable---in particular, the blue curve lies below the black dotted line, showing that subbagged regression trees satisfy the theoretical guarantee of \cref{thm:upper}. By contrast, the red curve lies much higher in the plot, showing greater instability; this reveals that the base algorithm, a regression tree with a maximum depth of $50$, is highly unstable. 

\section{Discussion and Related Work}

In this section, we first discuss some important practical implications of algorithmic stability. Next, we compare our main question to a prior work attempting to certify stability using hypothesis testing \citep{kim2021black}. Finally, we situate our work in the broader literature on the stability of bagging, and give some concluding remarks on the implications of this work.

\subsection{The Importance of Stability}\label{sec:importance}

Stability guarantees are central in a variety of contexts, despite the fact that many widely-used practical algorithms are not stable \citep{xu2011sparse}. For instance, \citet{bousquet2002stability} establish generalization bounds for stable learning algorithms, and \citet{mukherjee2006learning} show that stability is necessary and sufficient for empirical risk minimization to be consistent; related works include \citep{poggio2004general,kutin2012almost,freund2004generalization}. \citet{shalev2010learnability} identify stability as a necessary and sufficient condition for learnability. Stability is further relevant to differential privacy guarantees; assuming worst-case stability (often called ``sensitivity'' in the privacy literature) is a standard starting point for constructing differentially private algorithms \citep{dwork2008differential}. In the field of conformal prediction, distribution-free coverage guarantees rely upon the stability of the underlying estimators \citep[e.g.,][]{steinberger2016leave,steinberger2023conditional,ndiaye2022stable,barber2021predictive}. We now discuss applications of algorithmic stability to generalization and conformal inference in greater detail. 

\subsubsection{Stability and Generalization}

In a landmark work, \citet{bousquet2002stability} greatly expand our understanding of the connection between stability and generalization. In their telling, what distinguishes algorithmic stability from the pervasive uniform convergence theory is the following: whereas the latter aims to control the complexity of the space of learning rules an algorithm~$\ca$ searches over, the former emphasizes how the algorithm explores that space. Algorithmic stability notably first emerged as an invaluable tool to obtain generalization bounds for $k$-nearest neighbors \citep{rogers1978finite}, for which the underlying function class has unbounded complexity. For algorithms like bagging and nearest neighbors, where the strongest (nontrivial) guarantees hold for out-of-sample stability, the empirical risk is not necessarily reflective of test error and instead generalization holds with respect to the leave-one-out error---that is, the average leave-one-out error is a provably accurate estimate of the expected prediction error,
\[
\frac{1}{n}\sum_{i=1}^n \ell\left(\hat{f}^{\setminus i}\left(X_i\right), Y_i\right)\approx \E\left[\ell\left(\hat{f}\left(X\right),Y\right)\right],
\]
where the expected value is taken with respect to a new draw of $\left(X,Y\right)$ while treating $\hat{f}$ as fixed.
For an example of how random hypothesis stability (covered in Corollary~\ref{cor:loss-stability}) leads to polynomial bounds on the generalization error, see \citet[][Theorem 9]{elisseeff2005stability}.

\subsubsection{Predictive Uncertainty Quantification}

Algorithmic stability also plays an important role in the problem of predictive uncertainty quantification. Suppose $\left(X_i, Y_i\right)_{i=1}^{n+1}$ are \iid draws from an unknown distribution $P$, and $Y_{n+1}$ is unobserved. We wish to construct a prediction interval $\hat{C}_{n,\alpha} = \hat{C}_{n,\alpha}\left(X_{n+1}\right)$ (based on the training data $\cd = \left(X_i, Y_i\right)_{i=1}^n$, test covariate $X_{n+1}$ and learning algorithm $\ca$) that has guaranteed predictive coverage, that is,
\begin{align}
    \P\left\{Y_{n+1} \in \hat{C}_{n,\alpha}\left(X_{n+1}\right)\right\} \ge 1-\alpha,
\end{align}
without any restrictions on $\ca$ or $P$. If we wish to use an interval of the form $\hat{C}_{n,\alpha}\left(x\right) = \left[\hat{f}\left(x\right)-\hat{c}, \hat{f}\left(x\right)+\hat{c}\right]$, centered at the learning algorithm's prediction $\hat{f}\left(x\right)$, a natural approach to calibrating the radius $\hat{c}$ is to use the leave-one-out errors $R_i := \left|\hat{f}^{\setminus i}\left(X_i\right) - Y_i\right|$ as representative of the test error $\left|\hat{f}\left(X_{n+1}\right) - Y_{n+1}\right|$. This leads naturally to the classical leave-one-out technique known as the jackknife:
\[
\hat{C}_{n,\alpha}^{\textnormal{Jack}}\left(x\right)
:= \left[\hat{f}\left(x\right)-\hat{c}_\alpha, \hat{f}\left(x\right)+\hat{c}_\alpha\right], 
\]
where $\hat{c}_\alpha := Q_{1-\alpha}\left(\left\{R_i\right\}_{i=1}^n\right)$ is the $1-\alpha$ quantile of the leave-one-out errors $R_i$. 

If the base algorithm $\ca$ is \emph{unstable}, the leave-one-out errors need not be representative of test error at all. In fact, \citet{barber2021predictive} construct a pathological example for which that the jackknife has no coverage, that is,
\[
\P\left\{Y_{n+1} \in \hat{C}^{\textnormal{Jack}}_{n,\alpha}\left(X_{n+1}\right)\right\} = 0.
\]
\citet{barber2021predictive} go on to show that if the base algorithm $\ca$ is $\left(\varepsilon, \delta\right)$-stable, then coverage can be restored by inflating the radius to $\hat{c}_\alpha' := \hat{c}_\alpha + \varepsilon$ and running the procedure at level $\alpha' = \alpha - 2\sqrt{\delta}$. 

\subsection{Is Bagging Needed for Stability?}

Various learning algorithms are known to possess stability guarantees, such as $k$-nearest neighbors \citep{rogers1978finite,devroye1979distribution}, some regularized regression methods such as ridge regression \citep{bousquet2002stability,wibisono2009sufficient}, and models trained with stochastic gradient descent under smoothness assumptions \citep{hardt2016train}. Restricting to algorithms that are theoretically known to be stable can be quite limiting and can sacrifice accuracy in many settings. 

We might instead ask whether it is possible to validate empirically that an algorithm $\ca$ is stable with respect to a given data generating distribution. However,
\citet{kim2021black} show that it is essentially impossible to construct powerful hypothesis tests certifying $\left(\varepsilon, \delta\right)$-stability, without imposing assumptions on the algorithm or on distribution of the data. In their framework, we observe \iid random variables $\cd = \left(Z_i\right)_{i=1}^N$ where $Z_i\simiid P$. We wish to construct a test $\hat{T}$ that returns an answer 1 if we are confident that $\ca$ is $\left(\varepsilon, \delta\right)$-stable, or a 0 otherwise. Suppose we require that $\hat{T}$ obeys the following constraints:
\begin{enumerate}
    \item[(a)] $\hat{T}$ satisfies a universal bound on falsely declaring stability, that is, $\P\left\{\hat{T}=1\right\}\le \alpha$ for any $\ca$ that is \emph{not} $\left(\varepsilon, \delta\right)$-stable (with respect to distribution $P$ and sample size $n$), and
    \item[(b)] $\hat{T}$ is a \emph{black-box} test \citep[see][Definition 2]{kim2021black}, roughly meaning that $\hat{T}$ is only constructed using zeroth order oracle access to the algorithm $\ca$. That is, we may base our accept/reject decision on evaluating the model $\ca$ on data $\widetilde\cd$ that is simulated or resampled from the training data $\cd$, and compute predictions at test points $x$ that are generated similarly, an unlimited number of times.
\end{enumerate}
If a test $\hat{T}$ satisfies both properties (a) and (b) with no further assumptions on the distribution $P$ or on the algorithm $\ca$, their results imply that the power of $\hat{T}$ is upper bounded by
\[
\P\left\{\hat{T}=1\right\}
\le  \left(1-\delta\right)^{-N/n} \alpha,
\]
 for any $\ca$ that \emph{is} $\left(\varepsilon, \delta\right)$-stable (with respect to distribution $P$ and sample size $n$).
In particular, any universally valid black-box test has low power, unless the available data set size $N$ is far larger than the sample size $n$ for which we want to test stability.

In light of this impossibility result, a natural question is whether it is possible to \emph{convert} any algorithm $\ca$ into an $\left(\varepsilon, \delta\right)$-stable algorithm $\widetilde\ca$. Our work establishes the possibility of \emph{black-box stabilization}, that is, guaranteeing some quantifiable level of stability with no knowledge of the inner workings of the base algorithm. Our results support the use of bagging in such settings by certifying a certain level of $\left(\varepsilon, \delta\right)$-stability.

\subsection{Prior Work on the Stability of Bagging}\label{sec-related}

\citet{buhlmann2002analyzing} suggest (sub)bagging is most successful as a smoothing operation, softening hard threshold rules. They measure the instability of a procedure by its asymptotic variance:~$\ca$ is stable at~$x\in\cx$ if~$\hat{f}\left(x\right) \stackrel{p}{\to} f\left(x\right)$ as~$n\to \infty$, for some fixed~$f$. For some hard thresholding rules, they show bagging can reduce asymptotic variance. See also \cite{buja2000smoothing,friedman2007bagging}.

\citet{grandvalet2004bagging, grandvalet2006stability} exposes some limitations of the variance-reduction perspective. In particular, bagging need not reduce variance, and in simple examples its improvement over the base procedure need not relate to the original procedure's variance. 
\citeauthor{grandvalet2004bagging} illustrates through experiments a robustness property of bagging: highly influential data points are systematically de-emphasized. The role of~$\p$ in our main result, \cref{thm:upper}, underscores \citeauthor{grandvalet2004bagging}'s observation that the main stabilizing effect of bagging comes from the removal of high-leverage data points from a certain fraction of bags.

\citet{elisseeff2005stability} generalize standard notions of algorithmic stability~\citep{bousquet2002stability} to randomized algorithms and study (sub)bagging in this context. We can directly compare Corollary~\ref{cor:loss-stability} to the work of \citet[][Proposition 4.4]{elisseeff2005stability}, who also study the random hypothesis stability of subbagging with respect to an~$L$-Lipschitz loss~$\ell$. Their result shows subbagging satisfies condition~\eqref{eq:def-oos-prediction-stability} at the level
\begin{align}\label{eq:improving-stability}
    \beta = Lp\beta_{\ca,m},
\end{align}
where $\beta_{\ca,m}$ denotes the random hypothesis stability of the base algorithm~$\ca$ on data sets of size $m$ with respect to~$\ell_1$ loss. A similar result (under stronger assumptions) was obtained earlier by~\citet{poggio2002bagging}. 

We can interpret this result in two ways. First, if the base algorithm~$\ca$ is stable, the guarantee~\eqref{eq:improving-stability} suggests that bagging maintains or improves upon stability \citep[similar results have been obtained for boosting; see, e.g.,][]{kutin2001interaction}. Generally, we expect the stability of the base algorithm to improve with the sample size (i.e., $\beta_{\ca,m} \ge \beta_{\ca,n}$ for $m \le n$), so~\eqref{eq:improving-stability} does not necessarily imply subbagging improves upon the stability of running the base algorithm $\ca$ on the full data set. Second, the result of \citet{elisseeff2005stability} shows that we can achieve random hypothesis stability $\beta=O\left(n^{-1/2}\right)$ by taking $p=O\left(n^{-1/2}\right)$. By contrast, Corollary~\ref{cor:loss-stability} shows subbagging even half the data ($p=0.5$) can achieve random hypothesis stability~$\beta=O\left(n^{-1/2}\right)$.

\citet[Theorem 5]{chen2022debiased} consider subbagging when~$m=o\left(\sqrt{n}\right)$ and with \iid data. Specializing their result to the case of learning algorithms with bounded outputs, they guarantee worst-case stability at the level~$\varepsilon = o\left(n^{-1/2}\right)$ as long as $B\gg n$. By contrast, our result does not require \iid data, and  gives a faster rate $\varepsilon = o\left(n^{-3/4}\right)$ for fixed $\delta$ and $m=o\left(\sqrt{n}\right)$ (as well as results for larger $m$, e.g., for $m=O\left(n\right)$).

\subsection{Conclusion}

Distribution-free uncertainty quantification 
yields principled statistical tools which input black-box machine learning models and produce predictions with statistical guarantees, such as distribution-free prediction or calibration. Assumption-free stability is an important addition to this list, with a number of practical implications. Our work establishes assumption-free stability for bagging applied to any base algorithm with bounded outputs. These results suggest several avenues for future investigations, including formalizing lower bounds for distribution-free, black-box stabilization and characterizing the (sub)optimality of bagging. 

\nocite{devroye1979distribution2}

\section*{Acknowledgements}

RFB was supported by the National Science Foundation via grants DMS-1654076 and DMS-2023109, and by the Office of Naval Research via grant N00014-20-1-2337. JAS was supported by NSF DMS-2023109. RW was supported by NSF DMS-2023109, AFOSR FA9550-18-1-0166, NSF DMS-AWD00000326 and Simons Foundation MP-TMPS-00005320.

\appendix
\section{Proofs}\label{appendix-proofs}

This section contains proofs of all theoretical results from the main paper.

\subsection{Proof of Theorem~\ref{thm:upper}}

We abbreviate predicted values using $\hat{y} \coloneqq \hat{f}_\infty\left(x\right)$ and $\hat{y}^{\setminus i} \coloneqq
\hat{f}_\infty^{\setminus i}\left(x\right)$. Define
\[\mathcal{K} \coloneqq \left\{i \in [n] : \left|\hat{y} - \hat{y}^{\setminus i}\right| > \varepsilon\right\},\]
 the set of data points with large leave-one-out perturbation, and
let $K = \left|\mathcal{K}\right|$.
From Definition~\ref{def:average-case-stability}, we want to show~$K\le n\delta$.
Summing all the inequalities defining~$\mathcal{K}$ gives
\[
K\varepsilon \le \sum_{i\in \mathcal{K}}\left|\hat{y} - \hat{y}^{\setminus i}\right| =: L_1\left(\mathcal{K}\right).
\]
We now bound the error~$L_1\left(\mathcal{K}\right)$ on the right-hand side.

\emph{Step 1: Simplifying the leave-one-out perturbation.} For any bag~$r\in \seqn$, denote the value of our prediction using data~$\cd_r$ by \[\hat{y}^{\left(r\right)}:= \E_\xi [\ca\left(\cd_r; \xi\right)\left(x\right)].\]  
The aggregate prediction $\hat{y}$ can be expressed as
\[ \hat{y} =\E_{r\sim \cq_n,\xi}\big[\ca\left(\cd_r; \xi\right)\left(x\right)\big] =\E_{r\sim \cq_n,\xi}\big[ \hat{y}^{\left(r\right)}\big],\]
while $\hat{y}^{\setminus i}$ can be expressed as
\[ \hat{y}^{\setminus i} =\E_{r\sim \cq_{n-1},\xi}\left[\ca\left(\left(\cd^{\setminus i}\right)_r;\xi\right)\left(x\right)\right] =\E_{r\sim \cq_n}\big[\hat{y}^{\left(r\right)} \,\big\vert\, i\not\in r\big],\]
where the last step holds by symmetry (Assumption~\ref{assumption:symmetry}). Using the definition of conditional expectation, we have 
\begin{align*}
\hat{y}-\hat{y}^{\setminus i}
&= \E_{r \sim \cq_n}\left[\hat{y} - \hat{y}^{\left(r\right)} \,\middle|\, i\not\in r\right] \\
&= \frac{1}{1-p}\E_{r \sim \cq_n}\left[\left(\hat{y} - \hat{y}^{\left(r\right)}\right)\mathbf{1}\left\{i\not\in r\right\}\right].
\end{align*}

\emph{Step 2: Expressing $L_1\left(\mathcal{K}\right)$ as an expectation.} Define
\[
s_i := 
\textnormal{sign}\left(\hat{y}-\hat{y}^{\setminus i}\right) \cdot\mathbf{1}_{i\in \mathcal{K}}.\]
For each $i\in \mathcal{K}$, $\left|\hat{y}-\hat{y}^{\setminus i}\right|=s_i\left(\hat{y}-\hat{y}^{\setminus i}\right)$, so by Step 1, 
\begin{align}\label{eq:L1-expectation}
    L_1\left(\mathcal{K}\right)
=\frac{1}{1-p}\E_{r\sim\cq_n}\left[\left(\hat{y}-\hat{y}^{\left(r\right)}\right)\sum_is_i\mathbf{1}_{i\not\in r}\right].
\end{align}
\emph{Step 3: Bounding $L_1\left(\mathcal{K}\right)$.} Since $\hat{y}^{\left(r\right)}$ has mean $\hat{y}$, we may rewrite the right-hand side of~\eqref{eq:L1-expectation} as
\[
L_1\left(\mathcal{K}\right)
=\frac{1}{1-p}\E_{r\sim\cq_n}\left[\left(\hat{y}-\hat{y}^{\left(r\right)}\right)\left(\sum_is_i\mathbf{1}_{i\not\in r} - \E\left[\sum_is_i\mathbf{1}_{i\not\in r}\right]\right)\right].
\]
Applying Cauchy--Schwarz,
\begin{align}\label{eq:Cauchy-Schwarz} 
L_1\left(\mathcal{K}\right)
&\le \frac{1}{2\left(1-p\right)}\sqrt{\textnormal{Var}\left(\sum_is_i\mathbf{1}_{i\not\in r}\right)},
\end{align}
where we have used $\textnormal{Var}\left(\hat{y}^{\left(r\right)}\right)\le \frac{1}{4}$, known as Popoviciu's inequality, which uses $\hat{y}^{\left(r\right)}\in [0,1]$. We calculate the other variance term as 
\begin{align*}
    \textnormal{Var}\left(\sum_is_i\mathbf{1}_{i\not\in r}\right)
    &= p\left(1-p\right)\sum_i s_i^2 - q\sum_{i\ne j} s_is_j  \\
    &= \left(p\left(1-p\right) + q\right)\sum_i s_i^2 - q\Big(\sum_{i} s_i\Big)^2 \\
    &\le K\left(p\left(1-p\right) + q\right),
\end{align*}
since $q\ge 0$ and $\sum_i s_i^2 = \sum_i \mathbf{1}_{i\in \mathcal{K}} = K$. 
Combining everything,
\begin{align*}
K\varepsilon &\le \frac{1}{2\left(1-p\right)}\sqrt{K\left(p\left(1-p\right) + q\right)}. 
\end{align*}
Choosing~$\left(\varepsilon, \delta\right)$ to satisfy~\eqref{eq:eps-delta} implies~$K\le n\delta$. 

\subsection{Proof of Theorem~\ref{thm:upper-random}} Let $\hat{f}_B$ denote the result of bagging~$\ca$ on~$\cd$ with~$B$ bags, and similarly~$\hat{f}_B^{\setminus i}$ is the result of bagging~$\ca$ on~$\cd^{\setminus i}$ with~$B$ bags. We want to show
\begin{align*}
        \frac{1}{n}\sum_{i=1}^n\P_{\pmb{\xi}, \pmb{r}}\left\{\left|\hat{f}_B\left(x\right) - \hat{f}_B^{\setminus i}\left(x\right)\right| > \varepsilon + \sqrt{\frac{2}{B}\log\frac{4}{\delta'}}\right\}\le \delta+\delta',
\end{align*}
where~$\pmb{\xi}=\left(\xi^{\left(1\right)},\dots,\xi^{\left(B\right)}\right)$ and~$\pmb{r}=\left(r^{\left(1\right)},\dots,r^{\left(B\right)}\right)$ capture the randomness in the algorithm~$\ca$ and bagging, respectively. By the triangle inequality and union bound,
\begin{align*}
\frac{1}{n}\sum_{i=1}^n&\P_{\pmb{\xi}, \pmb{r}}\left\{\left|\hat{f}_B\left(x\right) - \hat{f}_B^{\setminus i}\left(x\right)\right| > \varepsilon + \sqrt{\frac{2}{B}\log\frac{4}{\delta'}}\right\} \\
&\le \frac{1}{n}\sum_{i=1}^n\P_{\pmb{\xi}, \pmb{r}}\left\{\left|\hat{f}_B\left(x\right) - \hat{f}_\infty\left(x\right)\right| >\sqrt{\frac{1}{2B}\log\frac{4}{\delta'}}\right\} \\
&~~~~~+ \frac{1}{n}\sum_{i=1}^n\mathbf{1}\left\{\left|\hat{f}_\infty\left(x\right) - \hat{f}_\infty^{\setminus i}\left(x\right)\right| > \varepsilon\right\} \\
&~~~~~+ \frac{1}{n}\sum_{i=1}^n\P_{\pmb{\xi}, \pmb{r}}\left\{\left|\hat{f}_\infty^{\setminus i}\left(x\right) - \hat{f}_B^{\setminus i}\left(x\right)\right| >\sqrt{\frac{1}{2B}\log\frac{4}{\delta'}}\right\}. 
\end{align*}
By \cref{thm:upper} the middle term on the right-hand side is at most~$\delta$. By Hoeffding's inequality,
\[
\P_{\pmb{\xi}, \pmb{r}}\left\{\left|\hat{f}_B\left(x\right) - \hat{f}_\infty\left(x\right)\right| > \sqrt{\frac{1}{2B}\log\frac{4}{\delta'}}\right\}
\le \frac{\delta'}{2}
\]
and similarly for~$\left|\hat{f}_B^{\setminus i}\left(x\right) - \hat{f}_\infty^{\setminus i}\left(x\right)\right|$.

\subsection{Proof of Theorem~\ref{thm:lower}} Let~$K = 1+\lfloor\delta n\rfloor$ and~$\eta = \frac{K}{n}>\delta$. For any positive integer~$m$, define
\[
\ca^\sharp\Big(\left(X_1,Y_1\right),\ldots,\left(X_m, Y_m\right)\Big)\left(x\right)
\coloneqq \mathbf{1}\left\{\sum_{i=1}^m X_i > \frac{mK}{n}\right\}.
\]

Define~$\cd = \{\left(x_i,y_i\right)\}_{i=1}^n$ where~$x_i = 1$ for $i\le K$ and~$x_i=0$ for $i > K$. For a bag~$r$ consisting of $m$ indices sampled without replacement, the algorithm $\ca^\sharp$ therefore returns the prediction $\hat{y}\left(r\right) = \mathbf{1}\left\{\sum_{i\in r}X_i > \frac{mK}{n}\right\}$. 
Let~$\hat{y}$ denote the average prediction, i.e., the result of subbagging $\ca^\sharp$, and let~$\hat{y}^{\setminus i}$ denote the average prediction over bags excluding~$i$.
It suffices to show that~$\left|\hat{y} - \hat{y}^{\setminus i}\right| >\varepsilon$ for each $i$ with~$x_i = 1$, since we then have
\[
\frac{1}{n}\sum_{i=1}^n \mathbf{1}\left\{\left|\hat{y} - \hat{y}^{\setminus i}\right| > \varepsilon\right\} \ge \frac{K}{n} = \eta > \delta,
\]
verifying that $\ca^\sharp$ fails to be $\left(\varepsilon, \delta\right)$-stable.
Let~$\hat{y}^i$ denote the average over all bags containing~$i$. Then for any~$i\leq K$ and $j>K$ (i.e., ~$X_i = 1$ and~$X_j=0$), by symmetry we can calculate
\begin{align*}
\hat{y} &= \E_{r\sim\cq_n}\left[\hat{y}\left(r\right)\right]\\
&=\frac{K}{n}\E_{r\sim\cq_n}\left[\hat{y}\left(r\right)\mid i_1\leq K\right] + \frac{n-K}{n}\E_{r\sim\cq_n}\left[\hat{y}\left(r\right)\mid i_1>K\right]\\
&=\eta\E_{r\sim\cq_n}\left[\hat{y}\left(r\right)\mid i_1=i\right] + \left(1-\eta\right)\E_{r\sim\cq_n}\left[\hat{y}\left(r\right)\mid i_1=j\right]
\\
&=\eta\E_{r\sim\cq_n}\left[\hat{y}\left(r\right)\mid i\in r\right] + \left(1-\eta\right)\E_{r\sim\cq_n}\left[\hat{y}\left(r\right)\mid j\in r\right]\\
&=\eta\hat{y}^i + \left(1-\eta\right)\hat{y}^j.\end{align*}
Similarly, we have
\begin{align*}
\hat{y} &= \E_{r\sim\cq_n}\left[\hat{y}\left(r\right)\right]\\
&=p\E_{r\sim\cq_n}\left[\hat{y}\left(r\right)\mid i\in r\right] + \left(1-p\right)\E_{r\sim\cq_n}\left[\hat{y}\left(r\right)\mid i\notin r\right]\\
& = p\hat{y}^i + \left(1-p\right)\hat{y}^{\setminus i}.\end{align*}
Combining these calculations,
\[
\hat{y} -\hat{y}^{\setminus i}
 =\frac{p}{1-p}\left(\hat{y}^i - \hat{y}\right)= \frac{p}{1-p}\left(1-\eta\right)\left(\hat{y}^i-\hat{y}^j\right).
\]
Similarly, noting that $\P\left\{j\in r \mid i\in r\right\} = \frac{m-1}{n-1}$, we have
\[
\hat{y}^i = \frac{m-1}{n-1}\hat{y}^{ij} + \frac{n-m}{n-1}\hat{y}^{i\setminus j},
\]
where~$\hat{y}^{ij}$ averages $\hat{y}\left(r\right)$ over bags containing both~$i$ and~$j$, and similarly $\hat{y}^{i\setminus j}$ averages over bags containing~$i$ and not~$j$. Similarly, $\hat{y}^j = \frac{m-1}{n-1}\hat{y}^{ij} + \frac{n-m}{n-1}\hat{y}^{j\setminus i}$, and therefore, $\hat{y}^i - \hat{y}^j = \frac{n-m}{n-1}\left(\hat{y}^{i\setminus j}-\hat{y}^{j\setminus i}\right)$. We write~$\hat{y}^{i\setminus j}$ and~$\hat{y}^{j\setminus i}$ as hypergeometric tail probabilities: 
\begin{align*}
    \hat{y}^{i\setminus j}
    &= \P_{H\sim \textnormal{HyperGeometric}\left(n-2, K-1, m-1\right)}\left\{1+H > \frac{mK}{n}\right\} \\
    \hat{y}^{j\setminus i}
    &= \P_{H\sim \textnormal{HyperGeometric}\left(n-2, K-1, m-1\right)}\left\{H > \frac{mK}{n}\right\}. 
\end{align*}
Combining our findings,
\begin{align*}
    \hat{y} -\hat{y}^{\setminus i}
    &= \frac{p\left(1-\eta\right)}{1-p}\frac{n-m}{n-1}\left(\hat{y}^{i\setminus j}-\hat{y}^{j\setminus i}\right) \\
    &= \frac{m\left(1-\eta\right)}{n-1}\P_{H\sim \textnormal{HyperGeometric}\left(n-2, K-1, m-1\right)}\left\{H = \left\lfloor \frac{mK}{n}\right\rfloor\right\} 
\end{align*}
Now let $h = \lfloor \frac{mK}{n}\rfloor$. Recalling $\eta = \frac{K}{n}$,
\begin{align*}
    \hat{y} -\hat{y}^{\setminus i}
    &=\frac{m\left(1-\eta\right)}{n-1}\P_{H\sim \textnormal{HyperGeometric}\left(n-2, K-1, m-1\right)}\left\{H = h\right\} \\
    &=\frac{m\left(1-\eta\right)}{n-1}\frac{{K-1\choose h}{n-K-1\choose m-h-1}}{{n-2 \choose m-1}} \\
    &=\frac{m\left(1-\eta\right)}{n-1}\frac{{K-1\choose h}{n-K\choose m-h}\frac{m-h}{n-K}}{{n-1 \choose m} \frac{m}{n-1}} \\
    &=\frac{m-h}{n}\frac{{K-1\choose h}{n-K\choose m-h}}{{n-1 \choose m}} \\
    &= \frac{m-\left\lfloor \frac{mK}{n}\right\rfloor}{n}\P_{H\sim \textnormal{HyperGeometric}\left(n-1, K-1, m\right)}\left\{H = \left\lfloor \frac{mK}{n}\right\rfloor\right\}\\
     &\geq \left(1-\delta-n^{-1}\right)p\P_{H\sim \textnormal{HyperGeometric}\left(n-1, K-1, m\right)}\left\{H = \lfloor p\left(1+\lfloor n\delta\rfloor\right)\rfloor\right\}\\
   & >\varepsilon,
\end{align*}
where the last step holds by assumption on $\varepsilon$. This verifies that $\left(\varepsilon, \delta\right)$-stability fails to hold, and thus completes the proof.

\subsection{Proof of Theorem~\ref{thm:unbounded-general}}
    This result follows from the proof of~\cref{thm:upper} if we substitute \cref{eq:Cauchy-Schwarz} with
    \[
     L_1\left(\mathcal{K}\right)
    \le \frac{R^*\left(\cd, x\right)}{2\left(1-p\right)}\sqrt{\textnormal{Var}\left(\sum_is_i\mathbf{1}_{i\not\in r}\right)}. 
    \]

\subsection{Proof of Theorem~\ref{thm:unbounded}} By the triangle inequality,
\begin{align*}
        \frac{1}{n}\sum_{i=1}^n&\mathbf{1}\left\{\left|\hat{f}_{\infty,I}\left(x\right) - \hat{f}_{\infty,I}^{\setminus i}\left(x\right)\right| > \varepsilon\,\textnormal{length}\left(I\left(\cd\right)\right)\right\} \\ 
        &\le \frac{1}{n}\sum_{i=1}^n\mathbf{1}\left\{\frac{\left|\hat{f}_{\infty,I}\left(x\right) - \hat{f}_{\infty,I}^{\setminus i}\left(x\right)\right|}{\cR\left(\cd, x\right)} > \varepsilon, I\left(\cd\right) = I\left(\cd^{\setminus i}\right)\right\}\\
        &~~~~~+\frac{1}{n}\sum_{i=1}^n\mathbf{1}\left\{I\left(\cd\right)\ne I\left(\cd^{\setminus i}\right)\right\}.
\end{align*}
It suffices to show that the first term on the right-hand side is at most~$\delta$.

Let $I_0 := I\left(\cd\right)$ denote the interval based on the full data set. Define a new algorithm $\ca^*$ that clips the output to $I_0$ regardless of the input data set. That is, for any data set $\cd'$ and test point $x\in \cx$, 
\[
\ca^*\left(\cd'\right)\left(x\right) := \frac{\E_\xi\left[\textnormal{Clip}_{I_0}\left(\ca\left(\cd';\xi\right)\left(x\right)\right)\right] - \inf I_0}{\textnormal{length}\left(I_0\right)}.
\]
This modified base algorithm has bounded outputs---that is, $\ca^*\left(\cd'\right)\left(x\right)\in [0,1]$. Let $\hat{f}_\infty^* = \widetilde\ca^*_\infty\left(\cd\right)$ denote the result of derandomized bagging (\cref{alg:bagging-derand}) on the modified base algorithm~$\ca^*$, and similarly let $\hat{f}_\infty^{*\setminus i} = \widetilde\ca^*_\infty\left(\cd^{\setminus i}\right)$. For any $i\in [n]$, on the event $I\left(\cd\right) = I\left(\cd^{\setminus i}\right)$, we have
\[
\frac{\left|\hat{f}_{\infty,I}\left(x\right) - \hat{f}_{\infty,I}^{\setminus i}\left(x\right)\right|}{\cR\left(\cd, x\right)}
= \left|\hat{f}_\infty^{*}\left(x\right) - \hat{f}_\infty^{*\setminus i}\left(x\right)\right|.
\]
Hence, applying \cref{thm:upper} to the modified base algorithm~$\ca^*$, 
\[
\frac{1}{n}\sum_{i=1}^n\mathbf{1}\left\{\frac{\left|\hat{f}_{\infty,I}\left(x\right) - \hat{f}_{\infty,I}^{\setminus i}\left(x\right)\right|}{\cR\left(\cd, x\right)} > \varepsilon, I\left(\cd\right) = I\left(\cd^{\setminus i}\right)\right\}
\le \frac{1}{n}\sum_{i=1}^n\mathbf{1}\left\{\left|\hat{f}_\infty^{*}\left(x\right) - \hat{f}_\infty^{*\setminus i}\left(x\right)\right| > \varepsilon\right\}
\le \delta,
\]
completing the proof.

\subsection{Proof of Theorem~\ref{thm:worst}} \emph{(i)} Fix~$\ca, \cd, x$, and let~$\hat{f}_\infty^i=\E_{r\sim \cq_n, \xi}\left[\hat{f}^{\left(r\right)} \mid i\in r\right]$. Observe that $\hat{f} = p\hat{f}_\infty^i + \left(1-p\right)\hat{f}_\infty^{\setminus i}$, so 
\[
\left|\hat{f}_\infty\left(x\right) - \hat{f}_\infty^{\setminus i}\left(x\right)\right|
= p\left|\hat{f}_\infty^i\left(x\right) - \hat{f}_\infty^{\setminus i}\left(x\right)\right|
\le p
\]
for all $i$.
This proves $\left(p,0\right)$-stability, and therefore, $\left(p,\delta\right)$-stability holds for all $\delta$.

\emph{(ii)} Let~$\ca^\dagger\left(\cd\right)\left(x\right) = \mathbf{1}\left\{\exists \left(\tilde{x}, \tilde{y}\right)\in \cd : \tilde{x} = x\right\}$---i.e., the algorithm checks whether~$x$ belongs to the training bag. Take~$\cd = \left(x_i, y_i\right)_{i=1}^n$ such that every~$x_i$ is unique. Then, for each $i$, $\hat{f}_\infty^{\setminus i}\left(x_i\right) = 0$, whereas $\hat{f}_\infty\left(x_i\right) =\p$.

\subsection{Proof of Corollary~\ref{cor:expectation}} Integrating Hoeffding's inequality,
\[
\E_{\pmb{\xi}, \pmb{r}}\left|\hat{f}_B\left(x\right) - \hat{f}_\infty\left(x\right)\right|
\le \int_0^12\exp\left(-2Bt^2\right)\text{d}t
\le \sqrt{\frac{\pi}{2B}},
\]
and similarly for $\left|\hat{f}_B^{\setminus i}\left(x\right) - \hat{f}_\infty^{\setminus i}\left(x\right)\right|$.

By \cref{thm:general-lp} (given below),
\begin{align*}
    &\frac{1}{n}\sum_{i=1}^n\left|\hat{f}_\infty\left(x\right) - \hat{f}_\infty^{\setminus i}\left(x\right)\right| 
    \le \sqrt{\frac{1}{4n}\left(\frac{\p}{1-\p} + \frac{\q}{\left(1-\p\right)^2}\right)}.
\end{align*}
Combining these bounds via the triangle inequality completes the proof.

\subsection{Proof of Corollary~\ref{cor:loss-stability}} Since the loss is $L$-Lipschitz,
\begin{align*}
\E_{\left(Z_i\right)_{i=1}^{n+1}\simiid P, \pmb{\xi}, \pmb{r}}&\left|\ell\left(\hat{f}_B\left(X_{n+1}\right), Y_{n+1}\right) - \ell\left(\hat{f}_B^{\setminus i}\left(X_{n+1}\right), Y_{n+1}\right)\right| \\
&\le L\,\E_{\left(Z_i\right)_{i=1}^{n+1}\simiid P, \pmb{\xi}, \pmb{r}}\left|\hat{f}_B\left(X_{n+1}\right) - \hat{f}_B^{\setminus i}\left(X_{n+1}\right)\right| \\
&= L\,\E_{\left(Z_i\right)_{i=1}^{n+1}\simiid P}\left[\frac{1}{n}\sum_{j=1}^n\E_{\pmb{\xi}, \pmb{r}}\left|\hat{f}_B\left(X_{n+1}\right) - \hat{f}_B^{\setminus j}\left(X_{n+1}\right)\right|\right] 
\end{align*}
The result follows upon applying Corollary~\ref{cor:expectation}.

\section{Stability Guarantee in \texorpdfstring{$\ell_k$}{General Lk Norms}}

In the main text, we establish guarantees for both worst-case and average-case stability. These two notions can be viewed as the $\ell_\infty$ and $\ell_1$ norms (respectively) of the sequence of leave-one-out perturbations~$\left(\left|\hat{f}\left(x\right) - \hat{f}^{\setminus i}\left(x\right)\right|\right)_{i=1}^n$. Our next result interpolates between these two settings by providing a guarantee for the~$\ell_k$ norm for any~$k > 0$. We state the result in the derandomized case for simplicity.

\begin{theorem}\label{thm:general-lp} In the setting of \cref{thm:upper}, define
    \[
    C\left(\cq_n\right) := \min\left\{\sqrt{\frac{1}{4n}\left(\frac{p}{1-p} + \frac{q}{\left(1-p\right)^2}\right)}, p\right\}.
    \]
Suppose $C\left(\cq_{n}\right)\le p$. Then, for any $k > 0$, derandomized bagging~$\widetilde{\ca}_\infty$ satisfies
\[
\left(\frac{1}{n}\sum_{i=1}^n\left|\hat{f}_\infty\left(x\right) - 
\hat{f}_\infty^{\setminus i}\left(x\right)\right|^k\right)^{1/k}
\le C\left(\cq_{n}\right)^{2/\max\left\{k, 2\right\}}p^{1-2/\max\left\{k,2\right\}}
\]
\end{theorem}

This result interpolates between some of our main stability guarantees. For instance, as $k\to\infty$,~\cref{thm:general-lp} yields
    \[
    \max_{i=1,\ldots,n} \left|\hat{f}_\infty\left(x\right) - 
\hat{f}_\infty^{\setminus i}\left(x\right)\right|
        \le p,
    \]
    recovering part \textit{(i)} of \cref{thm:worst}. Corollary~\ref{cor:expectation} covers the special case $k=1$
     \[
        \frac{1}{n}\sum_{i=1}^n\left|\hat{f}_\infty\left(x\right) - 
\hat{f}_\infty^{\setminus i}\left(x\right)\right|
        \le C\left(\cq_{n}\right),
    \]
    in the derandomized setting~$\left(B\to\infty\right)$. Finally, setting $k=2$, by Markov's inequality,
    \[
    \frac{1}{n}\sum_{i=1}^n \mathbf{1}\left\{\left|\hat{f}_\infty\left(x\right) - \hat{f}_\infty^{\setminus i}\left(x\right)\right| \ge \varepsilon\right\} \le \frac{C^2\left(\cq_{n}\right)}{\varepsilon^2},
    \]
    recovering \cref{thm:upper}. \\

\begin{proof} We first prove the result for $k=2$. We use the same notation as in the proof of \cref{thm:upper}, additionally defining $L_i = \left|\hat{y} - 
\hat{y}^{\setminus i}\right|$ and $s_i = \text{sign}\left(\hat{y} - 
\hat{y}^{\setminus i}\right)$. Following the same line of reasoning as in the proof of \cref{thm:upper}, 
\begin{align*}
    \|\Vec{L}\|_2^2
    &= \sum_{i=1}^n s_iL_i \cdot\left(\hat{y}-\hat{y}^{\setminus i}\right) \\
    &= \sum_{i=1}^n s_iL_i \cdot\mathbb{E}_r\left[\hat{y} - 
\hat{y}^{\left(r\right)}\big| i\not\in r\right] \\
    &= \mathbb{E}_r\left[\frac{1}{1-p}\sum_{i=1}^n s_iL_i \cdot\left(\hat{y} - 
\hat{y}^{\left(r\right)}\right)\mathbf{1}_{i\notin r}\right] \\
    &= \mathbb{E}_r\left[\left(\hat{y} - 
\hat{y}^{\left(r\right)}\right)\frac{1}{1-p}\sum_{i=1}^n s_iL_i \cdot\left(\mathbf{1}_{i\notin r} - p\right)\right] \\
    &\le \sqrt{\textnormal{Var}_r\left[\hat{y}^{\left(r\right)}\right] \cdot \textnormal{Var}_r\left[\frac{1}{1-p}\sum_{i=1}^n s_iL_i \cdot\left(\mathbf{1}_{i\notin r} - p\right)\right]} \\
    &\le \frac{1}{2\left(1-p\right)}\sqrt{\textnormal{Var}_r\left[\sum_{i=1}^n s_iL_i \mathbf{1}_{i\notin r}\right]}.
\end{align*}
Expanding the variance term,
\begin{align*}
    \textnormal{Var}\left(\sum_is_iL_i\mathbf{1}_{i\not\in r}\right)
    &= p\left(1-p\right)\sum_i s_i^2L_i^2 - q\sum_{i\ne j} s_iL_is_jL_j  \\
    &= \left(p\left(1-p\right) + q\right)\sum_i s_i^2L_i^2 - q\Big(\sum_{i} s_iL_i\Big)^2 \\
    &\le \left(p\left(1-p\right) + q\right)\sum_{i}L_i^2.
\end{align*}
After some rearranging, we have $\sqrt{\frac{1}{n}\sum_{i=1}^n L_i^2} \le \sqrt{\frac{1}{4n}\left(\frac{p}{1-p} + \frac{q}{\left(1-p\right)^2}\right)}.$ By part \textit{(i)} of \cref{thm:worst}, we have $\sqrt{\frac{1}{n}\sum_{i=1}^n L_i^2}\le C\left(\cq_{n}\right)$. Since $\left(\frac{1}{n}\sum_{i=1}^n L_i^k\right)^{1/k}$ is monotone in $k$, this also implies the result for $k\le 2$. For $k > 2$, we again use $\sqrt{\frac{1}{n}\sum_{i=1}^n L_i^2}\le C\left(\cq_{n}\right)$ and $\max_i L_i\le p$:
\[
\left(\frac{1}{n}\sum_{i=1}^n L_i^k\right)^{1/k}
\le \left(\frac{1}{n}\sum_{i=1}^n L_i^2 p^{k-2}\right)^{1/k}
\le C\left(\cq_{n}\right)^{2/k}p^{1-2/k},
\]
completing the proof.
\end{proof}

\section{Unbounded Outputs with Finite \texorpdfstring{$B$}{B}}\label{sec-finite-B-unbdd}

In this section, we present analogous results to \cref{thm:unbounded-general,thm:unbounded} for the finite~$B$ case.

\begin{theorem}\label{thm:finite-B-unbdd} Let~$\pred=\R$. Fix a distribution~$\mathcal{Q}_n$ on~$\seqn$ satisfying Assumptions~\ref{assumption:symmetry} and~\ref{assumption:nondegeneracy}, and let $\cq_{n-1}$ be defined as in~\eqref{eqn-define-Qn-Qn1}. Let $\left(\varepsilon, \delta\right)$ satisfy \cref{eq:eps-delta} and fix $\delta'>0$. For any algorithm~$\ca$, generic bagging~$\widetilde{\ca}_B$ is $\left(\varepsilon+\sqrt{\frac{2}{B}\log\left(\frac{4}{\delta'}\right)}, \delta+\delta', \bar{\cR}\right)$-stable, where 
\begin{align}
\bar{\cR}\left(\cd, x\right) :=\sup_{r\in \seqn, \xi\in\left[0,1\right]}\ca\left(\cd_{r}; \xi\right)\left(x\right)- \inf_{r\in \seqn, \xi\in\left[0,1\right]}\ca\left(\cd_{r}; \xi\right)\left(x\right).
\end{align}
\end{theorem}

\cref{thm:finite-B-unbdd} is proved the same way as \cref{thm:upper-random}, where we apply \cref{thm:unbounded-general} instead of \cref{thm:upper}. Next, we present our result for adaptively clipped bagging in the finite-$B$ regime.

\begin{theorem}
    Let~$\pred=\R$. Fix a distribution~$\mathcal{Q}_n$ on~$\seqn$ satisfying Assumptions~\ref{assumption:symmetry} and~\ref{assumption:nondegeneracy}, and let $\cq_{n-1}$ be defined as in~\eqref{eqn-define-Qn-Qn1}. Suppose the mapping to intervals~$I$ satisfies
    \begin{align}
    \frac{1}{n}\sum_{i=1}^n \mathbf{1}\left\{I\left(\cd\right)\ne I\left(\cd^{\setminus i}\right)\right\} \le \delta_I.
    \end{align}
    Let $\cR\left(\cd, x\right) = \textnormal{length}\left(I\left(\cd\right)\right)$, let $\left(\varepsilon, \delta\right)$ satisfy \cref{eq:eps-delta} and fix $\delta' > 0$. For any algorithm~$\ca$, adaptively clipped bagging $\widetilde{\ca}_{B,I}$ is $\left(\varepsilon + \sqrt{\frac{2}{B}\log\left(\frac{4}{\delta'}\right)}, \delta_I + \delta + \delta', R\right)$-stable.
\end{theorem}

\begin{proof} As in the proof of \cref{thm:unbounded}, 
\begin{align*}
        \frac{1}{n}\sum_{i=1}^n&\,\P\left\{\left|\hat{f}_{B,I}\left(x\right) - \hat{f}_{B,I}^{\setminus i}\left(x\right)\right| > \varepsilon\,\cR\left(\cd, x\right)\right\} \\ 
        &\le \delta_I + \frac{1}{n}\sum_{i=1}^n\P\left\{\frac{\left|\hat{f}_{B, I}\left(x\right) - \hat{f}_{B,I}^{\setminus i}\left(x\right)\right|}{\cR\left(\cd,x\right)} > \varepsilon, I\left(\cd\right) = I\left(\cd^{\setminus i}\right)\right\}.
\end{align*}
From this point on, following the same arguments as in the proof of Theorem 13, completing the proof via an application of \cref{thm:upper-random}.
\end{proof}

\bibliographystyle{apalike} 
\bibliography{reference}

\end{document}